\documentclass[11pt]{article}
\usepackage{tikz}
\tikzset{
  treenode/.style = {circle,draw, align=center,
                     top color=white, bottom color=blue!20},
  special/.style     = {treenode, font=\ttfamily\normalsize, bottom color=red!30},
  dummy/.style      = {treenode, font=\ttfamily\normalsize}
}
\usepackage[numbers]{natbib}

\usepackage{url}    
\usepackage{bbm}
\usepackage{booktabs}       
\usepackage{amsfonts}       
\usepackage{nicefrac}       
\usepackage{enumitem}
\usepackage{ltablex}
\usepackage{multirow}
\usepackage{mathrsfs}
\usepackage{times}
\usepackage{amsmath, amssymb, amsthm, graphicx, bm}
\usepackage[ruled,vlined,noresetcount]{algorithm2e}

\usepackage{mwe}
\usepackage[section]{placeins}
\usepackage{enumitem}
\usepackage{nicefrac} 
\usepackage{mathtools}

\usepackage{thmtools,thm-restate}
\usepackage[colorlinks, linkbordercolor=red,citebordercolor=green]{hyperref}
\usepackage{cleveref}

\newcommand{\E}{\mathbb{E}}
\newcommand{\R}{\ensuremath{\mathbb{R}}}

\newcommand{\argmax}{\mathop{{}\textrm{argmax}}}

\newcommand{\alg}{\ensuremath{\mathcal{A}}}
\newtheoremstyle{definition}
  {}
  {}
  {\itshape}
  {}
  {\bfseries}
  {.}
  { }
  {\thmname{#1}\thmnumber{ #2}\thmnote{ (#3)}}

  \newtheoremstyle{theorem}
  {}
  {}
  {\itshape}
  {}
  {\bfseries}
  {.}
  { }
  {\thmname{#1}\thmnumber{ #2}\thmnote{ (#3)}}

\theoremstyle{theorem}
\newtheorem{theorem}{Theorem}[section]
\newtheorem{lemma}[theorem]{Lemma}

\newtheorem{claim}[theorem]{Claim}
\newtheorem{proposition}[theorem]{Proposition}
\theoremstyle{definition}
\newtheorem{definition}[theorem]{Definition}

\newcommand{\abs}[1]{\lvert#1\rvert}

\newcommand{\Scal}{\ensuremath{\mathcal{S}}} 
 
\newcommand{\Acal}{\ensuremath{\mathcal{A}}}

\newcommand{\eps}{\ensuremath{\epsilon}}

\newcommand{\poly}{\textrm{poly}}
\DeclareMathOperator{\polylog}{polylog}
\DeclareMathOperator{\polyloglog}{polyloglog}
\newcommand{\dist}{\textrm{dist}}
\newcommand{\cnf}{\textsc{3-CNF}}
\newcommand{\sat}{\textsc{3-Sat}}

\newcommand{\gsat}{\textsc{Gap-$3$-Sat}}

\newcommand{\linear}[1]{\textsc{Linear-#1-RL}}
\newcommand{\klinear}{\textsc{Linear-}k\textsc{-RL}}
\newcommand{\reth}{\text{rETH}}
\newcommand{\npc}{\textsc{NP}}
\newcommand{\rpc}{\textsc{RP}}

\newcommand{\lp}{\left}
\newcommand{\rp}{\right}

\DeclareMathOperator{\ext}{ext}

\newcommand{\efix}[1]{\dist_{
S, \text{used}}({#1})}
\newcommand{\efree}[1]{\dist_{S, \text{free}}({#1})}
\newcommand{\efixc}[1]{\dist_{
S_{\text{curr}}, \text{used}}({#1})}
\newcommand{\efreec}[1]{\dist_{S_{\text{curr}}, \text{free}}({#1})}
\newcommand{\eqdef}{:=}
\newcommand{\PCP}{\text{PCP}}
\newcommand{\Z}{\mathbb Z}

\newcommand\snorm[2]{\left\| #2 \right\|_{#1}}

\usetikzlibrary{decorations.pathreplacing}
\tikzset{
  treenode/.style = {shape=circle,
                     draw, align=center, font=\scriptsize},
  root/.style     = {treenode},
  env/.style      = {treenode},
  leaf/.style     = {shape=circle,draw,align=center,font=\scriptsize},
  every node/.style       = {font=\tiny},
  dummy/.style    = {circle,draw,dashed,xshift=-3mm}
}

\usepackage{fullpage}
\usepackage{xcolor}
\usepackage{algorithmic}
\usepackage[margin=1in]{geometry}
\usepackage{environ}
\usepackage{mdframed}
\newmdenv[
  topline=false,
  bottomline=false,
  rightline=false,
  linewidth=0.8pt,
  skipabove=\topsep,
  skipbelow=\topsep,
  innertopmargin=2pt,
  innerbottommargin=0pt
]{siderules}

\NewEnviron{problem}[1]{%
\begin{center}\fbox{\parbox{3in}{%
    {\centering\scshape #1\par}%
    \parskip=1ex
    \everypar{\hangindent=1em}%
    \BODY
}}\end{center}}

\title{Exponential Hardness of Reinforcement Learning with Linear Function Approximation}
\author{%
  Daniel Kane\thanks{Supported by NSF Award CCF-1553288 (CAREER) and a Sloan
  Research Fellowship.}\\
  University of California, San Diego\\
  \texttt{dakane@eng.ucsd.edu}
    \and
    Sihan Liu\\
    University of California, San Diego\\
    \texttt{sil046@ucsd.edu}
    \and
    Shachar Lovett\thanks{Supported by NSF Awards DMS-1953928 and CCF-2006443.}\\
  University of California, San Diego\\
  \texttt{slovett@cs.ucsd.edu}
    \and
    Gaurav Mahajan\\
    University of California, San Diego\\ 
    \texttt{gmahajan@eng.ucsd.edu}
    \and 
    Csaba Szepesv\'ari\thanks{Supported by NSERC, Amii, and the Canada AI Research Chair program.}\\
DeepMind, London, UK\\
University of Alberta, Edmonton, Canada\\
\texttt{szepesva@ualberta.ca}
\and 
Gell\'ert Weisz\\
DeepMind, London, UK\\
University College London, London, UK\\
\texttt{gellert@deepmind.com}
}
\date{\today}
\makeatletter
\newcommand*{\rom}[1]{\expandafter\@slowromancap\romannumeral #1@}
\makeatother
\begin{document}
\maketitle
\newif\iftodo
\todotrue
\iftodo
\newcommand{\gm}[1]{\textsf{\color{magenta}{GM: #1}}}
\newcommand{\gellert}[1]{\textsf{\color{purple}{GW: #1}}}
\newcommand{\shachar}[1]{[\textsf{\color{red}{SL: #1}}]}
\newcommand{\sihan}[1]{[\textsf{\color{red}{Sihan: #1}}]}
\newcommand{\csaba}[1]{\textsf{\color{blue}{Cs: #1}}}
\else
\newcommand{\gm}[1]{\ignorespaces}
\newcommand{\gellert}[1]{\ignorespaces}
\newcommand{\shachar}[1]{\ignorespaces}
\newcommand{\sihan}[1]{\ignorespaces}
\newcommand{\csaba}[1]{\ignorespaces}
\fi

\begin{abstract}
A fundamental question in reinforcement learning theory is: suppose the optimal value functions are linear in given features, can we learn them efficiently? This problem's counterpart in supervised learning, linear regression, can be solved both statistically and computationally efficiently. Therefore, it was quite surprising when a recent work \cite{kane2022computational} showed a computational-statistical gap for linear reinforcement learning: even though there are polynomial sample-complexity algorithms, unless NP = RP, there are no polynomial time algorithms for this setting.

In this work, we build on their result to show a computational lower bound, which is exponential in feature dimension and horizon, for linear reinforcement learning under the Randomized Exponential Time Hypothesis. To prove this we build a round-based game where in each round the learner is searching for an unknown vector in a unit hypercube. The rewards in this game are chosen such that if the learner achieves large reward, then the learner's actions can be used to simulate solving a variant of 3-SAT, where (a) each variable shows up in a bounded number of clauses (b) if an instance has no solutions then it also has no solutions that satisfy more than (1-$\epsilon$)-fraction of clauses. We use standard reductions to show this 3-SAT variant is approximately as hard as 3-SAT. Finally, we also show a lower bound optimized for horizon dependence that almost matches the best known upper bound of $\exp(\sqrt{H})$.
\end{abstract}
\newpage
\tableofcontents
\newpage
\section{Introduction}
Efficiently exploring and planning in environments with large state spaces is a central problem in reinforcement learning. Recently, there has been a lot of success in applying function approximation to classical reinforcement learning algorithms leading to state-of-the-art results in various practical applications.

This has also led to a growing interest of the reinforcement learning (RL) theory community to design and analyze efficient algorithms for the large state space regime. In this regime, the goal is to design algorithms whose complexity does not polynomially depend on the size of the state space. Since, this is impossible when we do not make any assumptions about the environment, much effort has been spent on finding minimal assumptions under which an optimal policy can be found efficiently: State
Aggregation \citep{lihong2009disaggregation,dong2020provably}, 
Linear $q^\pi$ \citep{du2019good,lattimore2019learning,yin2022efficient,weisz2022confident}, 
Linear MDPs \citep{yang2019sample,jin2019provably}, Linear Mixture
MDPs \citep{modi2019sample,ayoub2020model,zhou2021provably}, Reactive
POMDPs \citep{krishnamurthy2016pac}, Block
MDPs \citep{du2019provably}, FLAMBE \citep{agarwal2020flambe},
Reactive PSRs \citep{littman2001predictive}, Linear Bellman
Complete \citep{munos2005error,zanette2020learning}, Bellman rank \citep{jiang2016contextual}, Witness rank \citep{sun2018model}, Bilinear Classes \citep{bilinear2021}, Bellman Eluder \citep{jin2021bellman} and Decision-Estimation Coefficient \citep{foster2021statistical}. 

One such minimal assumption that came out of this line of work is RL with linear function approximation: when the optimal value function (either $Q^*$, or $V^*$, or both) can be obtained as the linear combination of finitely many, known basis functions. 
When both the optimal value functions $Q^*$ and $V^*$ satisfy this assumption (called linear $Q^* \& V^*$ henceforth), 
there are two \emph{sample efficient} algorithms in the literature whose sample complexities are polynomial in the number of basis functions $d$ and horizon $H$. 
First, the algorithm by \cite{bilinear2021} additionally assumes that the basis functions' values can be known and pre-processed for the whole state-action space.
Second, TensorPlan \citep{weisz2021queryefficient,weisz2022tensorplan}
replaces this with an implicit assumption that the number of actions is a small constant (as its sample complexity is exponential in this number).
\cite{weisz2021exponential,weisz2022tensorplan} showed sample complexity lower bounds exponential in $\min(d,H)$ that imply statistical hardness of finding a near-optimal policy
when the number of actions is polynomial in $d$ and the values of basis functions are only revealed for the sampled states.
This indicates that one of the two aforementioned additional assumptions are required for a sample efficient algorithm.
However, even when both additional assumptions are met, these works leave finding a computationally efficient algorithm for this setting as an important open question.

A recent work \citep{kane2022computational} made progress on this question by showing a computational-statistical gap in RL with linear function approximation: unless NP=RP, there is no polynomial time algorithm even for the easiest setting of linear $Q^* \& V^*$, deterministic transition, stochastic rewards and 2 actions. This is surprising because if we also assume that the rewards are deterministic, then this problem can be solved in $O(dH)$ time \citep{wen2017efficient}. 
Therefore, the result of \citep{kane2022computational} showed that adding noise in rewards can lead to computational intractability (similar transition happens for sample complexity if the number of actions is unrestricted \cite{weisz2021exponential,weisz2022tensorplan}). However, the lower bound of \citep{kane2022computational} is not tight: they showed a quasi-polynomial lower bound in $d$ whereas the best known upper bounds are exponential in $\min(d,H)$ \citep{du2020agnostic}. 

\section{Our Contributions}
In this work, we provide almost matching exponential computational lower bounds for RL with linear function approximation.
Before stating our main results, we first need to state some key definitions that we use throughout the paper.

\subsection{Preliminaries}
\paragraph{Markov Decision Process (MDP).} We begin by defining the framework for reinforcement learning, a Markov Decision Process (MDP). We define a deterministic transition MDP as a tuple $M = \left(\Scal,\Acal, R, P\right)$, where $\Scal$ is the state space, $\Acal$ is the action space, $R:\Scal\times\Acal\mapsto \Delta([0,1])$ is the stochastic reward function,%
\footnote{$\Delta([0,1])$ denotes the set of all distributions over the interval $[0,1]$.}
and $P:\Scal\times\Acal\mapsto \Scal$ is the deterministic transition function. 
Such an MDP $M$ gives rise to a discrete time sequential decision process where an agent starts from a starting state $S_0\in \Scal$. Then, at each time $t$, the agent at some current state $S_t$, takes action $A_t$, receiving reward $R_t \sim R(S_t, A_t)$ and transitions to next state $S_{t+1}=P(S_t,A_t)$. This goes on until the agent reaches the end state $\bot\in \Scal$. In $H$-horizon problems
each such trajectory/path from the starting state $s_0$ to an end state $\bot$ is of length of at most $H$, 
and the sets of states $\Scal_t$ that are reachable after $t$ steps (taking any actions) are disjoint for $0\le t\le H$.
The goal of the decision making agent is to maximize the  sum of the total expected rewards it receives along such a trajectory. As it turns out, the total expected reward regardless the initial state is achievable by following
a deterministic, stationary policy, which is given by some map $\pi:\Scal \mapsto \Acal$ and following $\pi$ means that in step $t$ if the state is $S_t$, the action taken is $A_t = \pi(S_t)$.
Given a policy $\pi$ and a state-action pair $(s,a) \in \Scal \times \Acal$, we let
\begin{equation*}
    V^\pi(s)=\E\left[\sum_{t = 0}^{\tau-1}R(S_{t},A_{t})\mid S_0 =s, \pi\right],\quad Q^\pi(s,a) = \E\left[\sum_{t = 0}^{\tau - 1}R(S_{t},A_{t})\mid S_0 =s, A_0 = a, \pi\right]
\end{equation*}
denote the total expected reward 
    where $S_{1}, A_{1}, \ldots S_{\tau-1}, A_{\tau - 1}$ are obtained by executing policy $\pi$ in the MDP $M$ and $\tau$ is the first time when policy $\pi$ reaches the end state $\bot$, that is $S_\tau = \bot$ where it always holds that $\tau \le H$. 
    We use $Q^*$ and $V^*$ to denote the optimal value functions \[
        V^*(s) = \sup_{\pi} V^\pi(s)\, , \quad  Q^*(s,a) = \sup_{\pi} Q^\pi(s,a)\, , \quad s \in \Scal, a \in \Acal
    \] 
    We say that the optimal value functions $V^*$ and $Q^*$ can be written as a linear function of $d$-dimensional features $\psi \colon \Scal \sqcup (\Scal \times \Acal) \to \R^d$ if for all state $s$ and action $a$, $V^*(s) = \langle \theta, \psi(s)\rangle$ and $Q^*(s,a) = \langle \theta, \psi(s,a)\rangle$ for some fixed $\theta \in \R^d$ independent of $s$ and $a$.\footnote{Above, $\sqcup$ means taking the disjoint union of the arguments.} In our construction, linear $V^*$ implies linear $Q^*$ for $\psi(s,a) = \psi(P(s,a))$ as (i) in deterministic transition MDPs, $Q^*(s,a) = r(s,a)+ V^*(P(s,a))$, (ii) in our construction, rewards are $0$ everywhere except at the leaves and (iii) the reward at the leaves does not depend on the action. 

    \paragraph{Computational Problems.} We next introduce \sat, a satisfiability problem for \cnf~formulas. In a \sat~problem, we are given as input, a \cnf~formula $\varphi$ with $v$ variables and $O(v)$ clauses and our goal is to decide if $\varphi$ is satisfiable.

    \begin{center} \label{def:sat}
        \begin{minipage}{0.9\textwidth}
        \hrulefill\\[5pt]
        \textbf{Complexity problem}\quad \sat~ \\[3pt]
        \begin{tabular}{@{}p{0.1\linewidth} @{}p{0.89\linewidth}}
             \textit{Input:} &  A \cnf~formula $\varphi$ with $v$ variables and $O(v)$ clauses \\
             \textit{Goal:} & Decide whether the formula is satisfiable.
        \end{tabular} \\[5pt]
        \vphantom{.}\hrulefill
        \end{minipage}
    \end{center}

The focus of this work is the computational RL problem, \klinear.  In a \klinear~problem with feature dimension $d$, we are given access to a deterministic MDP $M$ with $k$ actions and horizon $H = O(d)$ such that the optimal value functions $Q^*$ and $V^*$ can be written as a linear function of the $d$-dimensional features $\psi$. Our goal is to output a good policy, which we define as any policy $\pi$ that satisfies $V^\pi > V^* - 1/8$, where  $V^\pi$ and $V^*$ refers to the value of the policy $\pi$ and optimal policy, respectively, at a fixed starting state and is always in $[0,H]$ \footnote{In our constructions, we satisfy the more stringent condition that $V^* \in [0,1]$.}.
From now on, we always assume that the number of actions is $k=3$.

\begin{center}
    \begin{minipage}{0.9\textwidth}
    \hrulefill\\[5pt]
    \textbf{Complexity problem}\quad \klinear\\[3pt]
    \begin{tabular}{@{}p{0.1\linewidth} @{}p{0.89\linewidth}}
         \textit{Oracle:} & a deterministic MDP $M$ with $k$ actions, optimal value functions $V^*$ and $Q^*$ linear in $d$ dimensional features $\psi$, horizon $H$ and state space of size at most $\exp( \poly(d) )$.\\
         \textit{Goal:}&  find policy $\pi$ such that  $V^\pi > V^* - 1/8$.
    \end{tabular} \\[5pt]
    \vphantom{.}\hrulefill
    \end{minipage}
    \end{center}

We now describe how the algorithm interacts with the MDP. We assume that the algorithm has access to the state and action spaces (which can be taken as subsets of integers), as well as random access to the associated 
(i) reward function $R$, 
(ii) transition function $P$ and 
(iii) features $\psi$. For all these functions, the algorithm provides a state $s$ and action $a$ (if needed) and receives a random sample from the distribution $R(s,a)$ (for the reward function), the state $P(s,a)$ (for the transition function), features $\psi(s)$ and $\psi(s,a)$ (for the features).
We assume that each call accrues constant runtime and input/output for these functions are of size polynomial in feature dimension $d$. 

We will often talk about randomized algorithm $A$ solving a problem in time $t$ with error probability $p$. By this we mean (i) $A$ runs in time $O(t)$; (ii) for satisfiability problems, it returns YES on positive input instances with probability at least $1-p$ and returns NO on negative input instances with probability $1$; and (iii) for an RL problem, it returns a good policy with probability at least $1-p$. 




\subsection{Exponential lower bound for \linear{3}}
In this paper, we present computational lower bound under a strengthening of the \npc~$\neq$ \rpc~conjecture, the Randomized Exponential Time Hypothesis (\reth) \citep{reth2014hardness}, which asserts that probabilistic algorithms can not decide if a given \sat~problem with $v$ variables and $O(v)$ clauses is satisfiable in sub-exponential time.

\begin{restatable}[Randomized Exponential Time Hypothesis (\reth)]{definition}{defreth}
    \label{def:reth}
  There is a constant $c > 0$ such that no \emph{randomized} algorithm can decide \sat~with $v$ variables in time $2^{c v}$ with error probability $1/3$. 
\end{restatable}

The Randomized Exponential Time Hypothesis along with many variants motivated by the Exponential Time Hypothesis \citep{eth2001} has been influential in discovering hardness results for a variety of problems see, e.g. \cite{cygan2015lower,vassilevska2019hardness}. Under the Randomized Exponential Time Hypothesis, our main result is an exponential computational lower bound for learning good policies in deterministic MDPs with linear optimal value functions.
\begin{restatable}[Exponential in horizon and dimension lower bound]{theorem}{thmmain}
    \label{thm:main}
    Under \reth, there is no randomized algorithm that solves \linear{3}~with feature dimension $d$ and horizon $H$ in time $\exp(\tilde O(\min(d^{1/4}, H^{1/4})))$ with probability at least $9/10$, where $\tilde O$ hides $\polylog(d)$ and $\polylog(H)$ factors. 
\end{restatable}

A few remarks are in order. Firstly, $\min(\cdot)$ is the correct complexity measure here. To see this, we note that this problem can be solved in time $\exp(\tilde O(\min(d,\sqrt{H})))$ (we prove these upper bounds in \Cref{appendix:upper})
and therefore if either dimension $d$ or horizon $H$ is constant, we can solve this problem efficiently in the other parameter. Secondly, this is the first exponential computational lower bound for this setting as the previous best known result \cite{kane2022computational} produces at best a quasi-polynomial lower bound, even assuming \reth.

In terms of horizon $H$, there is still a gap between the $\exp(\tilde \Omega(H^{1/4}))$ lower bound in \Cref{thm:main} and the $\exp(\tilde O(\sqrt{H}))$ upper bound. We next show a lower bound optimized for horizon $H$ which almost matches this upper bound. 

\begin{restatable}[Almost matching horizon lower bound]{theorem}{thmhmain}
  \label{thm:hmain}
  Under \reth, there is no randomized algorithm that solves \linear{3}~with horizon $H$ and feature dimension $d \geq H^{ \log H }$ in time $
  \exp(\tilde O({ \sqrt{H} }))$ with probability at least $9/10$, where $\tilde O$ hides $\polylog(H)$ factors.
\end{restatable}
We now discuss some open questions. Even though the lower bound in \Cref{thm:hmain} almost matches the upper bound in terms of horizon $H$, it requires the feature dimension to be at least quasi-polynomial in $H$. We leave it as an open question if the above result also holds when $d = \poly(H)$. Another important direction is understanding the complexity in terms of dimension $d$ i.e. a lower bound optimized for dimension $d$. Our proof for \Cref{thm:main} can be modified to show $\exp(d)$ lower bound for $H= \exp(d)$. Does the result also hold true for $H = \poly(d)$?

\paragraph{Related Work.}
We already discussed the large body of work giving statistical efficient algorithms for RL under various assumptions. Complementing them is work giving statistical lower bounds for RL with linear function approximation when the number of actions grows. Concretely, the works of \cite{weisz2021exponential,weisz2022tensorplan,wang2021exponential} showed sample complexity lower bounds exponential in $\min(d,H)$ that imply statistical hardness of finding a near-optimal policy,
when the number of actions grow with the number of basis functions and the values of basis functions are only revealed for the sampled states. Furthermore, there are recent works \citep{golowich2022learning,golowich2022planning,uehara2022computationally} on designing quasipolynomial-time end-to-end algorithm for learning in "observable" POMDPs (our lower bound result refute existence of similar quasipolynomial-time algorithms for linear $Q^*$ and $V^*$ assumption.)

\paragraph{Remainder of this paper.} In \Cref{sec:tech}, we present a brief overview of the main technical ideas in the lower bound construction. In \Cref{sec:lower}, we describe in detail our exponential lower bound constructions and prove our main theorems. In \Cref{app:hardness} we use standard reductions to show that under the randomized Exponential Time Hypothesis, a gap version of SAT that we use in the reduction is computationally hard. In \Cref{appendix:upper} we give algorithms for RL which are exponential in $\min(d,\sqrt{H})$, showing that our lower bound is close to optimal.

\section{Proof Overview}
\label{sec:tech}
The high-level idea of the previous lower bound of \cite{kane2022computational} was the following. The authors design an MDP that forces the learner to search for an unknown vector $w^*$ in $\{0,1\}^v$ which constitutes a satisfying solution of a given SAT formula $\phi$. 
In particular, each state in the MDP corresponds to an assignment and the learner at the state can flip one variable appearing in the first unsatisfying clause of the formula (assuming some canonical ordering of the clauses). Rewards are given when the learner either reaches a satisfying assignment or the end of the horizon.
The rewards are designed in such a way that (i) the learner is incentivized for finding $w^*$ quickly but (ii) unable to exploit much information from the rewards to accelerate the searching process. As a result, the task becomes as hard as solving the original SAT problem.

One bottleneck of the above approach is that the reward is only uninformative if the algorithm plays the game for fewer than \emph{quasi-polynomially} many times. After that, there is a decent chance that the algorithm could obtain extra information from the reward structure which may significantly simplify the task. We follow the same high level idea of embedding hard (variants of) SAT instances into a linear-RL problem. Yet, we make significant modifications to the transition and reward structure of the MDP such that the algorithm can hardly obtain any useful information from the rewards unless it plays the game for \emph{exponentially} many times.
In essence, we ensure the rewards given at the end of the horizon are uninformative by making it a Bernoulli variable with exponentially small mean. If so, the learner with high probability sees only $0$ in the end unless it plays the game for a large number of times.
As a warm-up, one could imagine an MDP with actions and transitions identical to that from \cite{kane2022computational}. 
Yet, we modify the reward to be $\exp(-$ number of steps thus far $- \dist(w,w^*))$ at any terminal state $w$.
This makes sure the (expected) reward given at the end of the horizon is always exponentially small.
Unfortunately, the value function induced will be of the same exponential function, and hence cannot be written as a linear function of some low-dimensional features depending only on the state.

\paragraph{Round Based Game.}
One way to fix this is by turning the game into a round based game. We divide the search into rounds and in each round, the variables are shown sequentially for the learner to decide whether to flip a variable it or not.
Then, if the learner terminates at the $n$-th round, we make the reward function roughly $\prod_{i=1}^n g_i($number of flips taken in round $i) \cdot g_{n+1}( \dist(w, w^*) )$ for some carefully chosen low-degree and monotonically decreasing polynomials $g_i$.
Now, consider the greedy policy which tries to decrease the distance to $w^*$ whenever possible.
Since the greedy policy can always reach $w^*$ within one (entire) round, the value function of such strategy at the beginning of round $i$ will be $\prod_{j<i} g_j($number of flips in round $j)*g_i(\dist(w,w^*))$. 
Since only the last term depends on $w$ and $w^*$, we get that the value function is essentially a 
low-degree
polynomial in $w$ and $w^*$, which can indeed be written as a linear function of some state-dependent low-dimensional feature vectors. See \Cref{lem:greedy-reward} for details.

However, in order to ensure that this is the optimal strategy, we will need to define the $g_i$ very precisely so that making a flip in the current round is always better than deferring it to future rounds. Essentially this means that the logarithmic derivative of $g_j$  should be smaller than the logarithmic derivative of $g_i$ for $j>i$. 
Ideally, we would like to make $g_i(x) = \exp(-c_i\,x)$ for some increasing sequence of $c_i$, which would then make the above property trivially true.
However, since $g_i$ must be a polynomial, we will instead make it a Taylor approximation to this exponential function around $x = 0$. 
As long as we can make the error in this Taylor approximation small relative to the difference in logarithmic derivatives of $\exp(-c_i \cdot x)$, 
it remains advantageous for the agent to take additional steps in earlier rounds. Fortunately, this is indeed achievable using a low-degree Taylor approximation.
See \Cref{clm:monotone} and \Cref{lem:greedy-optimality} for details of the argument.

\paragraph{Flips Enforcement.} 
While the round-based game does ensure the linearity of the value function, the reward given at the end of the horizon is not necessarily small. Since $g_i(x)$ is taken to be the Taylor approximation of $\exp(-c_i \cdot x)$ around $x=0$, $g_i(0)$ will be $1$. Consequently, if the learner chooses to flip nothing, it may receive a huge reward in the end, allowing the algorithm to extract information from the reward structure.  

To prevent this, we will offer the learner a bundle of variables in the first step of each round so that it must flip one of the given variables. A caveat of doing so is that we want at least one variable to be indeed erroneous so that flipping it results in the correct truth assignment to it and hence the greedy policy is still well-defined and optimal.
Fortunately, this is guaranteed if
we simply give the variables appearing in any of the unsatisfied clauses.

\begin{figure}
  \begin{center} 
    \begin{minipage}{0.9\textwidth}
    \hrulefill\\[5pt]
    \textbf{Complexity problem}\quad $(b,\eps)$-~\gsat~ \\[3pt]
    \begin{tabular}{@{}p{0.1\linewidth} @{}p{0.89\linewidth}}
         \textit{Input:} &  A \cnf~formula $\varphi$ with $v$ variables and $O(v)$ clauses with the following promise: (1) each variable is in at most $b$ clauses, and
        (2) either $\varphi$ is satisfiable or  any assignment leaves at least $\eps$-fraction of clauses unsatisfied. 
         \\
         \textit{Goal:} & Decide whether the formula is satisfiable.
    \end{tabular} \\[5pt]
    \vphantom{.}\hrulefill
    \end{minipage}
  \end{center}
\end{figure}

This allows us to force the algorithm to make at least one flip. In order to make the rewards diminish at a faster rate, we take the idea further: we keep presenting the learner with unsatisfied clauses involving variables that have not yet been flipped. Only after running out of such clauses, we start to go through the rest of the variables and give the learner the choice to skip flips.
\begin{figure}[htb]
\begin{center} 
\resizebox{0.7\columnwidth}{!}{
\begin{tikzpicture}
  [
    grow                    = right,
    level 1/.style          = {sibling distance=1.3cm},
    level 2/.style          = {sibling distance=1.3cm},
    level 3/.style          = {sibling distance=1.3cm},
    level distance          = 1.3cm,
    edge from parent/.style = {draw, edge from parent path={(\tikzparentnode) -- (\tikzchildnode)}},
    sloped
  ]
  \node (root)[anchor=west,root,label={[label distance=0.3cm,text depth=3ex,rotate=90]right:$w^{(i)}=(0,1,0,0,0)$}] {}
  child { node [anchor=west,root,label={[label distance=0.3cm,text depth=3ex,rotate=90]right:$w=(0,1,1,0,0)$}] {}
        child { node [dummy] {}
        edge from parent node [above] {$d=1$}
      }
      child { node [anchor=west,root,label={[label distance=0.3cm,text depth=3ex,rotate=90]right:$w=(0,1,1,0,1)$}] {}
        child { node [dummy] {}
            edge from parent node [above] {$a=1$}
        }
        child { node [anchor=west,root,label={[label distance=0.3cm,text depth=3ex,rotate=90]right:$$}] {}
          child { node [anchor=west,root,label={[label distance=0.3cm,text depth=3ex,rotate=90]right:$$}] {}
            child { node [dummy] {\color{red} B}
                edge from parent node [above] {$d=0$}
            }
            child { node [anchor=west,root,label={[label distance=0.3cm,text depth=3ex,rotate=90]right:$w^{(i+1)}=(0,1,1,1,1)$}] {\color{red} B}
              edge from parent node [above] {$d=1$}
            }
            edge from parent node [above] {$b=1$}
          }
          child { node [dummy] {\color{red} A}
              edge from parent node [above] {$b=0$}
          }
            edge from parent node [above] {$a=0$}
        }
          edge from parent node [above] {$e=1$}
      }
      child { node [dummy] {}
          edge from parent node [above] {$a=1$}
      }
      edge from parent node [above] {$c=1$}
    }
    child { node [dummy] {}
        edge from parent node [above] {$b=0$}
    }
    child { node [dummy] {}
      edge from parent node [above] {$a=1$}
    };
    \draw [decorate,decoration={brace,amplitude=4pt,mirror,raise=2pt},yshift=0pt]
(0,-2.8) -- (2.9,-2.8) node [black,midway,yshift=-10pt] {\scriptsize
Stage I};
\draw [decorate,decoration={brace,amplitude=4pt,mirror,raise=2pt},yshift=0pt]
(3.0,-2.8) -- (8.0,-2.8) node [black,midway,yshift=-10pt] {\scriptsize
Stage II};
\draw [decorate,decoration={brace,amplitude=4pt,mirror,raise=2pt},yshift=0pt]
(0,-3.3) -- (8.0,-3.3) node [black,midway,yshift=-10pt] {\scriptsize
Round $i$};
\end{tikzpicture}
}
\caption{Example mechanics of the MDP for round $i$. The MDP consists of $h$ consecutive rounds, of which only round $i$ is shown. Nodes are states with their assignment $w$ labeled where it changes, and edges are actions where the label represents the setting of some variable.
The satisfiability problem is $(a \vee \neg b \vee c) \wedge (c \vee d \vee e) \wedge (a \vee d \vee e) \wedge (a \vee \neg b \vee \neg c) \wedge (a \vee \neg b \vee \neg e)$, for variables $a$ to $e$ that have assignment of $w^{(i)}$ at the start of the round.
For illustrative simplicity, note that this problem does not belong to $(b,\eps)$-~\gsat.
The first two steps form Stage I as there is an unsatisfied clause consisting of only free variables.
The second stage allows to change any of the remaining free variables one by one.
Transitions are deterministic. Rewards are always zero except for termination conditions {\color{red}A} and {\color{red}B}, where the reward is Bernoulli. {\color{red}A}: the assignment satisfies at least $(1-\epsilon)$ fraction of clauses. {\color{red}B}: only if $i$ is the last round, the game is terminated at the end of the round.
} 
\label{fig:mdp-example}
\end{center}
\end{figure}
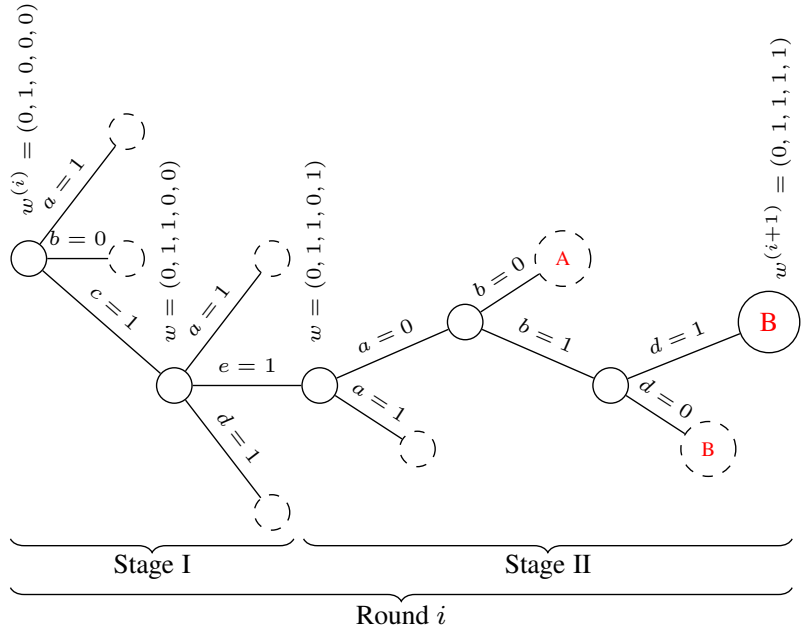


We would like to require that we can find many such clauses. Of course this is not possible to guarantee in a general SAT instance. However, we show there is a special family of \sat~instances so that finding assignments where one would quickly run out of such unsatisfied clauses is computationally hard. 
In particular, we use some standard reductions to show that \sat~is approximately as hard as what we call \gsat~where (a) each variable shows up in a bounded number of clauses (b) if there are no solutions then there are no solutions that satisfy a (1-$\epsilon$)-fraction of clauses. Note that (b) above implies that it is hard to find any assignment satisfying a (1-$\epsilon$)-fraction of clauses, and (a) says that flipping a variable can only remove a constant number of unsatisfied clauses from consideration. In particular, if each variable appears in at most $b$ clauses, then any computationally efficient algorithm will never run out of unsatisfied clauses in the first $\eps * ($total number of clauses$/b)$ steps.
This ensures that the reward at the horizon is exponentially small in the number of rounds.



\section{Lower Bound Construction}
\label{sec:lower}
In this section, 
we will prove the following computational lower bound for \linear{3} under \reth.
\begin{proposition}
    \label{prop:main}
    Let $v \in \mathbb Z^+$ be sufficiently large. Suppose $d, H \in \mathbb Z^+$ satisfy either
    \begin{enumerate}
        \item $d =  v^{4} \cdot \polylog(v)$ and  $H = \Theta( v^{4} )$, or
        \item $d = 
        \exp \lp(  \log^2 v \cdot \polyloglog(v) \rp)$ and $H = \Theta( v^{2} )$.
    \end{enumerate}
    Then, under rETH, no randomized algorithm can solve \linear{3} with feature dimension $d$ and time horizon $H$ in time $\exp( v / \polylog(v) )$ with error probability $1/10$.

\end{proposition}
Our main theorems, \Cref{thm:main} and \Cref{thm:hmain}, follow from \Cref{prop:main} by writing $v$ as a function of $d$ and $H$. 

\subsection{From \cnf~formulas to 3-action MDPs} \label{sec:mdp-construction}
Recall that in $(b, \eps)$-\gsat, we are given as input a 3-CNF formula $\varphi$ on $v$ variables where (1) each variable is guaranteed to occur in at most $b$ clauses and
(2) the formula is either satisfiable or any assignment satisfies at most $(1-\eps)$-fraction of the clauses (the formula is guaranteed to fall in one of these two cases). Furthermore, we may assume that the number of clauses is at least $v$. By \Cref{prop:gap-hardness}, we know deciding whether $\varphi$ is satisfiable must take time that is exponential in $v$ under \reth~when $b, \eps$ are set to be two absolute constants. Our goal is to construct an MDP parametrized by $\varphi$ so that learning a near optimal policy for the MDP is as hard as solving $(b, \eps)$-\gsat. When the formula is satisfiable, additionally the MDP will have an extra parameter $w^*$ which is chosen to be an arbitrary satisfying assignment of the formula.

To consolidate the two results in \Cref{prop:main}, in our reduction, we construct the MDP with two additional ``degree parameters'' $p, q\in \Z^+$. The MDP will have feature dimension $d= 2 \cdot v^{2p}$ and time horizon $H = \alpha \cdot v^q$ where $\alpha$ is a small enough absolute constant to be determined later. In particular, the hard instance for the first result is obtained by setting $p = 2, q = 4$ and the hard instance for the second result is obtained by setting
$p = 2 \log v, q = 2$.


\paragraph{State Action Transition.} 
The time steps are divided into $h := H/v$ rounds where each round consists of $v$ steps. 
In each round, the MDP maintains a set of ``used variables'', initialized to be the empty set at the beginning of each round. We will call unused variables as ``free variables''. One round is further divided into two stages as follows:
\begin{enumerate}
    \item
          In each step of the first stage, the agent is given one unsatisfied \sat~clause with only \emph{free variables} and asked which of the three variables should be flipped.
          Then, the variable chosen by the algorithm will be added to the set of used variables.
          The first stage ends (and the second stage starts) when we run out of unsatisfied clauses with only free variables.
    \item In the second stage, the MDP presents each of the remaining free variables sequentially to the agent and asks whether the variable should be flipped. After each step, regardless of whether the algorithm decides to flip the variable or not, the variable presented will be added to the set of used variables.
\end{enumerate}
Note that each round has exactly $v$ steps since there are $v$ variables in the formula and each step marks one of them as used.

\paragraph{Termination Condition.} The MDP terminates if it reaches the last level, or if more than a $(1 - \eps)$-fraction of the clauses are satisfied. We make a couple of observations related to the termination conditions. First, if the MDP terminates before reaching the last level, the algorithm has essentially solved the underlying \gsat~problem since this means there exists an assignment that satisfies more than $(1 - \eps)$-fraction of the clauses. Secondly, the termination condition ensures that, at the beginning of each round, there are at least an $\eps$-fraction of unsatisfied clauses. Since each variable appears in at most $b$ clauses, we will never run out of unsatisfied clauses with only free variables in the first $\eps \cdot ($total number of clauses$/b)$ steps.

\paragraph{Size of state space.}
The number of states necessary for a round and an assignment is at most $O(3^v)$: the transitions within a round form a tree of branching factor at most $3$ and height at most $v$ (see \cref{fig:mdp-example}).
There are $h=H/v$ rounds, 
the whole transition structure is a tree,
hence the total number of states is at most $O((3^v)^h)=\exp(\text{poly}(v,H))$.

\paragraph{Rewards.} Rewards are given only when the MDP terminates and are different depending on whether the formula is satisfiable or not. When the formula is not satisfiable, the reward is $0$ everywhere.
In the rest of the discussion, we will think of assignments as vectors in $\{-1, 1\}^v$. When the formula is satisfiable, we need to keep track of the assignment at the beginning of each round and denote them as $w^{(1)}, w^{(2)} \ldots, w^{(n)}$ (note that the algorithm starts at the first state with the assignment $w^{(1)}$), on which the final reward depends. The reward depends only on the history $w^{(1)}, w^{(2)} \ldots, w^{(n)}$, the current assignment
$w$ and the optimal assignment $w^*$ and is given by  $Ber(r(w^{(1)}, w^{(2)} \ldots, w^{(n)}, w))$ where $r(\cdot)$ is the expected reward function and the Bernoulli distribution $Ber( \rho )$ is $1$ with probability $\rho$ and $0$ with probability $1-\rho$. Before specifying the expected reward function $r(\cdot)$,
we introduce the concept of an \emph{extended assignment}.


\begin{definition}[Extended Assignment]
    \label{def:extended}
    Let $S$ be the set of free variables. Then, the extended assignment of $w$ under $S$, denoted as $\ext(w, S)$ is given by $\ext(w,S)_i = w^*_i$ for $i \in S$ and $\ext(w,S)_i = w_i$ for $i \not \in S$.
\end{definition}
In plain language, the extended assignment is the assignment derived from $w$ after correcting all the free variables to agree with $w^*$.
We note that dependence of the reward function on the extended assignment is crucial to ensure that the value functions associated to the greedy policy are linear functions, which will become relevant later. Now we are ready to define expected reward function $r(\cdot)$ in terms of the historic assignments $w^{(1)}, \cdots, w^{(n)}$ reached by the agent at the end of past rounds, the current assignment $w$ reached by the agent and the set of free variables $S$ when the MDP terminates.
\begin{definition}[Expected Reward]\label{def:exprew}
    Let $p,q \in \Z^+$ be the two degree parameters.
    Let $T_{p}: \R \mapsto \R^+$ be the degree-$p$ Taylor approximation of the exponential function $\exp(\cdot)$ at zero:
    $$
    T_p(x) = \sum_{i=0}^p \frac{x^i}{i!}.
    $$
    Then, we define expected reward function $r(w^{(1)}, w^{(2)} \ldots, w^{(n)}, w, S)$ as
\begin{align}
    \label{eq:expected-reward}
     \left(\prod_{i=1}^{n-1} g_i( \dist(w^{(i)}, w^{(i+1)})) \right)\cdot
    g_n( \dist(w^{(n)}, \ext(w, S))) \cdot g_{n+1}( \dist(\ext(w, S), w^*)),
\end{align}
where the polynomial $g_i: \R \mapsto \R$ for round $i$ 
is defined as
\begin{align} \label{eq:sub-reward-log}
    g_i(x) = T_{ p }\lp( - \frac{x}{ v^{q-1} \cdot (3 - i / h ) } \rp) \, .
\end{align}
\end{definition}
\noindent As noted in the proof overview, the polynomials $g_i$ are chosen to ensure that the optimal policy prefers going towards $w^*$ as fast as possible and using a low degree Taylor approximation ensures the value function for the optimal policy can be written as a linear function of low dimensional features. 

\subsection{Linear Value Function} 
When the underlying formula is unsatisfiable, any policy is optimal since the reward is constantly $0$. When the formula is satisfiable, we will show that the ``greedy policy'' is optimal.
\begin{definition}[Greedy Policy]\label{def:greedy}
    We say a policy is greedy if at every state it chooses any action that decreases the distance to $w^*$ whenever possible. If not, it tries to not increase the distance to $w^*$.
\end{definition}
 Notice that based on our setup of the MDP greedy policies exist: 
 in the first stage of a round, the algorithm is given an unsatisfied clause so there is at least one variable in the clause that can be flipped to decrease the distance from the current assignment to $w^*$; in the second stage, the algorithm is given variables one at a time and it can always choose to not flip the variable if the current assignment already agrees with $w^*$ on the variable.

We first discuss the value function $V^{\pi}$ associated to a greedy policy $\pi$. Given a state  with current assignment $w$ and a set $S$ of free variables, we define the following concepts that will be useful in the discussion. Let $m(S) \in \{0,1\}^v$ be the masking vector such that $m(S)_i = 1$ if the $i$-th variable is in $S$ and $m(S)_i = 0$ otherwise. Moreover, let $\mathbf 1$ denote the all-one vector and $\circ$ the point-wise multiplication operator. Then, we define
\begin{align*}
    \dist_{S, \text{free}}(w, w^*) &= \dist( w \circ m(S) , w^* \circ m(S) ) \\
    \dist_{S, \text{used}}(w, w^*) &= \dist( w \circ (\mathbf 1  - m(S)) , w^* \circ (\mathbf 1 -m(S)) )
\end{align*} 
In other words, $\dist_{S, \text{used}}(w, w^*)$ and $\dist_{S, \text{free}}(w, w^*)$ are the number of used and free variables respectively where the current assignment differs from $w^*$. Note that $\dist_{S, \text{used}}(w, w^*) + \dist_{S, \text{free}}(w, w^*) = \dist(w, w^*)$. 

Moreover, since the Hamming distance $\dist(a,b)$ for two vectors $a, b \in \{-1, 1\}^v$ 
is linear in both $a$ and $b$ (as  $\dist(a, b) = (v - \langle a, b\rangle)/2$), this implies $\efree{w, w^*}$ and $\efix{w, w^*}$ can be written as a linear function of $w^*$ and some state specific parameters depending on the current assignment $w$ and the set of free variables $S$ only. This allows us to show that the value functions for the greedy policy can also be written as linear functions of $w^*$ and some state specific parameters.
\begin{lemma}
    \label{lem:greedy-reward}
    When $\varphi$ is satisfiable, the greedy policy's value at state $s$ with round history $w^{(1)}, \ldots, w^{(n)}$, current assignment $w$ and the set of free variables $S$,  is given by
    \begin{align}
        \label{eq:greedy-reward}
        V^{\pi}(s) = \prod_{i=1}^{n-1} g_i( \dist(w^{(i)}, w^{(i+1)} ) ) \cdot g_n(\dist(w^{(n)}, w) + \efree{w, w^*} ) \cdot g_{n+1}(\efix{w, w^*}).
    \end{align}
    As a result, there exists features $\psi(s), \psi(s,a) \in \R^d$ with feature dimension $d \le 2 v^{2p}$ depending only on state $s$ and action $a$; and $\theta \in \R^d$ depending only on $w^*$ such that $V^\pi$ and $Q^\pi$ can be written as a linear function of features $\psi$ i.e. $V^\pi(s) = \langle \theta, \psi(s)\rangle$ and $Q^\pi(s,a) = \langle \theta, \psi(s,a)\rangle$.
\end{lemma}
\begin{proof}
    The first claim follows from the fact that the greedy policy will choose an action that will decrease the distance between the current assignment and the optimal assignment $w^*$ used by the MDP whenever there is such an action. As a result, starting from a state $s$, it will flip all the free variables where $w$ and $w^*$ differ in the current round,
    and then flip all the used variables where $w$ and $w^*$ differ in the next round. Upon reaching $w^*$, the final reward received will be exactly Equation \eqref{eq:greedy-reward} with no intermediate rewards.
    
    Following the greedy policy may fail to reach $w^*$. The only way this can happen is when the MDP terminates early: when more than $(1-\eps)$-fraction of the clauses is satisfied, or when we reached a final state in the last round. In such cases,
    the reward received depends on the extended assignment of the terminal state.
    From Definitions~\ref{def:extended} and \ref{def:exprew} it follows that the reward received is the same than the reward would have been if the MDP were not to terminate at that point. Hence, the reward received is still consistent with Equation~\eqref{eq:greedy-reward}.

    To prove the second claim, we follow a similar approach as in the proof of Proposition 10 in \cite{kane2022computational}.
    In particular, we will show that $V^\pi(s)$ can be written as a polynomial of degree at most $2p$ in $w^*$.
    To see why this is enough, we set $\theta$ to be all monomials in $w^*$ of degree at most $2p$. That is, each coordinate of $\theta$ corresponds to a multiset $S \subset [v]$ of size $|S| \le 2p$, and its value is $\theta_S = \prod_{i \in S}w^*_i$.
    We set $\psi(s)$ to be the corresponding coefficients in the polynomial $V^\pi$. Then, we can write $V^\pi(s) = \langle \theta, \psi(s)\rangle$.
    Since, there are at most $\sum_{i=0}^{2p} v^i \leq 2 v^{2p}$ many coefficients we can set the feature dimension as $d = 2 v^{2p}$.

    Finally, we prove that $V^\pi(s)$ can be written as a polynomial of degree at most $2p$ in $w^*$. First recall that $\efree{w, w^*}$ and $\efix{w, w^*}$ can be written as a linear function of $w^*$ and some state specific parameters $w^{(n)}, w$ and $S$. Moreover, $\dist(w^{(n)}, w)$ is independent of $w^*$ and only depends on $w^{(n)}$ and $w$. Then the fact is proven by noting for each $g_i(\cdot)$ in the expression that: (i) for $i<n$ it is independent of $w^*$; and (ii) for $i \in \{n,n+1\}$, it is a degree-$p$ polynomial in $\dist(w^{(n)}, w)$, $\efree{w, w^*}$ and $\efix{w, w^*}$. 
    
    Finally, note that linear $V^\pi$ implies linear $Q^\pi$ in deterministic MDPs for $\psi(s,a) = \psi(P(s,a))$, since by definition, in MDPs with deterministic transition, $Q^\pi(s,a) = r(s,a)+V^\pi(P(s,a))$
    and the rewards in our MDPs are zero, except for the last stage where the rewards do not depend on the action.
\end{proof}

We now prove some structural properties of the polynomials $g_i$. First, we will show that if a policy makes a ``reasonable'' number of flips in a round, then the value function decreases by a multiplicative factor. This follows from $g_i$ being (an appropriate degree) Taylor approximation of $\exp(\cdot)$ function around zero.
\begin{claim}
    \label{clm:range}
    The polynomials $g_i$ defined in Equation \eqref{eq:sub-reward-log} are bounded:
    \begin{equation*}
        \frac{1}{4} \leq g_i(x) \leq 1 - \frac{\eps}{6bv^{q-2}}
    \end{equation*}
  for all $ \frac{\eps}{b} \cdot v \leq x \leq v$ and $i \in \{1, \ldots ,h\}$.
  Moreover, $g_i$ is monotonically decreasing.
  
\end{claim}
\begin{proof}
    For simplicity let \[
        z =  \frac{x}{ v^{q-1} \cdot (3 - i/h) }.
    \] 
    For the range of values of $x$ we are interested in and since $q \geq 2$, it follows that 
    $z \leq 1/2$. Then, using the fact that $g_i$ is a Taylor approximation, we can upper bound $g_i$ by
    \begin{align}
        g_i(x) = T_p\lp( - z \rp) =
        \sum_{j=0}^{p}
         \frac{(-z)^j}{j!} 
               & \leq 1 - \frac{z}{2} \tag{as $p \geq 2$ and $z \leq 1/2$}\\
        &\leq 1 - \frac{\eps}{6b v^{q-2}}\,. \tag{as $x \geq \frac{\eps}{b} \cdot v$}
    \end{align}
    On the other hand, we can lower bound $g_i(x)$ as follows:
    \begin{align} \label{eq:geometric}
        g_i(x) = \exp \lp( - z \rp)
        - \sum_{j=p+1}^{\infty}
        \frac{\lp( - z \rp)^j}{j!}&\geq 
        \exp \lp( - \frac{1}{2} \rp)
        - \frac{1}{2^{p}(p+1)!}
        \geq \frac{1}{4}\,,
    \end{align}
    where the first inequality again follows from $z
 \leq 1/2$ and summation of geometric series and the last inequality holds as long as $p \geq 1$.
    
    Next, we argue $g_i(x)$ is monotonically decreasing. We do so by showing the derivative of $g_i(x)$ is negative. For this, we calculate
    \begin{align*}
        \frac{d}{dx} g_i(x)
        = 
        \lp( \frac{d}{dx} z \rp)
        \cdot
        \lp( \frac{d}{dz} T_p(-z) \rp)
        = 
         - \sum_{j=0}^{p-1} \lp( \frac{ (-z)^{j} }{j!} \rp).
    \end{align*}
    Similar to Equation~\eqref{eq:geometric}, we have
    $$
     \sum_{j=0}^{p-1} \lp( \frac{ (-z)^{j} }{j!} \rp)
     \geq  
     \exp \lp( - \frac{1}{2} \rp)
        - \frac{1}{2^{p-1}p!} > 0
    $$
    whenever $p \geq 2$.
    Therefore, $\frac{d}{dx} g_i(x) < 0$ which implies that $g_i(x)$ is monotonically decreasing.
\end{proof}
Next, we will show that the polynomials are designed such that correcting variables (where $w$ and $w^*$ differ) in round $i$ is always better than correcting variables in round $i+1$. 
In particular, suppose we have flipped $c$ bits in the $i$-th round and $d$ bits in the $(i+1)$-th round. We then want to show that 
$g_i(c) \cdot g_{i+1}(d) 
\geq g_i(c-1) \cdot g_{i+1}(d+1)$ for any $1\leq c \leq v $ and $0 \leq d \leq v$.
To prove this, we need to show that the error from Taylor approximation which depends on the choice of $p$ is relatively small.
\begin{claim} \label{clm:monotone}
    For any two polynomials $g_i, g_{i+1}$ defined in Equation \eqref{eq:sub-reward-log}, let
    \begin{equation*}
        f_{i,c,d}(x) = g_i(c + x) \cdot g_{i+1}(d - x).
    \end{equation*} where $i \in \{1, \ldots, h\}$, $0 \leq c,d \leq v$ and $x = \{1,2,\ldots,d\}$. 
    Then, for large enough $v$, \begin{equation*}
        f_{i,c,d}(x) \geq f_{i,c,d}(x-1).
    \end{equation*}
\end{claim}
\begin{proof}
    Consider the function $\hat f_{i,c,d}(\cdot)$ defined as
    \begin{align}
        \label{eq:f-def}
        \hat f_{i,c,d}(x)
        =
        \exp\lp(  - \frac{c + x }{ v^{q-1} \cdot (3- i/h) } \rp) \cdot
        \exp\lp( - \frac{d - x}{ v^{q-1} \cdot (3- (i+1)/h ) } \rp).
    \end{align}
    To prove our claim, we will show that \begin{align}
        \hat f_{i,c,d}(x) -  \hat f_{i,c,d}(x-1) &\geq \Omega \lp( \frac{1}{\alpha \cdot v^{2q-2}} \rp)\label{eq:topr2}\, ,\\
        \abs{ f_{i,c,d}(x) - \hat f_{i,c,d}(x)} &= O\lp( \frac{1}{v^{2q-2}} \rp) \label{eq:topr1}\ \, ,
    \end{align}
    where to recall $\alpha$ in Equation~\eqref{eq:topr2} is the parameter in the time horizon factor, i.e. $H = \alpha \cdot v^q$ and $h=H/v$.
    Then, our claim follows from the inequalities above as long as $\alpha$ is set to be a sufficiently small constant.
    We first prove \Cref{eq:topr2}. For this, we will show that the derivative of $\hat f_{i,c,d}(x)$ is not only positive but lower bounded by $\Omega(1/ \alpha \cdot v^{2q-2})$.
    The derivative of $\hat f(\cdot)$ is given by
    \begin{align} \label{eq:f-derivative}
        \hat f'_{i,c,d}(x)
        =
        \hat f_{i,c,d}(x)
        \cdot
        \frac{h}{ v^{q-1} \cdot \lp( 3 h - i \rp) \cdot \lp( 3h -i - 1 \rp) }.
    \end{align}
    Notice that
    we always have
    $$
        \hat f_{i,c,d}(x) \geq
        \exp
        \lp(
        - \frac{2v}{v^{q-1}}
        \rp)
        \cdot
        \exp \lp(
        - \frac{v}{v^{q-1}}
        \rp)
        = \exp \lp( - \frac{3}{v^{q-2}} \rp) \geq \Omega(1),
    $$
    where the first step follows from $0 \leq c,d,x \leq v$ and $i \in \{1, \ldots, h\}$ and last step from $q \geq 2$. We hence have $\hat f_{i,c,d}(x)  \geq \Omega(1)$.
    Combining this with $h := \alpha \cdot v^{q-1}$ and Equation \eqref{eq:f-derivative}, we can lower bound the derivatives by 
    \[
        \hat f'_{i,c,d}(x)
        \geq \Omega \lp( \frac{1}{\alpha \cdot v^{2q-2}} \rp).
    \]
    Since $\hat f_{i,c,d}$ is a convex function, this proves \Cref{eq:topr2}.
   \\

   \noindent Next, we prove \Cref{eq:topr1}. Recall that \[
        g_i(y) = T_p\lp(\frac{-y}{v^{q-1} \cdot (3- i/h)}\rp)
    \]
    where $T_p$ is the degree-$p$ Taylor approximation of the exponential function.
    Then, for $0 \leq y \leq 2v$ we have
    \begin{align} \label{eq:g-approx-diff}
        \lp |
        g_i( y ) - \exp \lp( -
        \frac{y}{v^{q-1} \cdot (3- i/h)}
        \rp)
        \rp |
        \leq O\lp (  \lp( \frac{y}{v^{q-1} \cdot (3- i/h)} \rp)^{p+1}   \rp) \leq O\lp( \frac{1}{v^{(q-2)\cdot (p+1)} \cdot 2^{p+1}} \rp).
    \end{align}
    In addition, for $y \geq 0$ we have
    \begin{align}
        \label{eq:exp-bound}
        \exp(-y / ( v^{q-1} \cdot (3- i/h) ) < 1 .
    \end{align}
    Substituting Equations \eqref{eq:g-approx-diff} and \eqref{eq:exp-bound} into Equation \eqref{eq:f-def} then gives
    \begin{align*}
        f_{i,c,d}(x)
         & =
        \lp( \exp\lp(  - \frac{c + x }{ v^{q-1} \cdot (3 - i/h) } \rp)
        \pm O\lp( \frac{1}{v^{(q-2)\cdot (p+1)} \cdot 2^{p+1}} \rp)
        \rp) \cdot                   \\
         & 
        \lp(
        \exp\lp( - \frac{d - x}{ v^{q-1} \cdot (3 - (i+1)/h ) } \rp)
        \pm O\lp( \frac{1}{v^{(q-2)\cdot (p+1)} \cdot 2^{p+1}} \rp)
        \rp)                   \\
         & = \hat f_{i,c,d}(x)
        \pm O\lp( \frac{1}{v^{(q-2)\cdot (p+1)} \cdot 2^{p+1}} \rp).
    \end{align*}
    For both settings of $p$ and $q$ we consider, $p=2; q =4$ for the first result or $p = 2 \log v; q = 2$ for the second result, this implies
    \begin{align*}
        \abs{ f_{i,c,d}(x) - \hat f_{i,c,d}(x) }
        \leq O\lp( \frac{1}{v^{(q-2)\cdot (p+1)} \cdot 2^{p+1}} \rp) = O\lp( \frac{1}{v^{2q-2}} \rp).
    \end{align*}
\end{proof}
Using above recursively, we can show that any greedy policy is an optimal policy in our MDPs. This is important because this in conjunction with \Cref{lem:greedy-reward} implies that the optimal value functions $V^*$ and $Q^*$ can be written as a linear function of some features depending only on states $s$ and action $a$.
\begin{lemma} \label{lem:greedy-optimality}
    Any greedy policy as defined in \Cref{def:greedy} is optimal.
\end{lemma}
\begin{proof}
    Let $\pi$ be the greedy policy and consider some other policy $\tilde \pi$. We show the reward received by the greedy policy $\pi$ is no worse than $\tilde \pi$ starting from an arbitrary state $s_{\text{curr}}$ with assignment $w_{\text{curr}}$, free variables $S_{\text{curr}}$ and round history $w^{(1)}, \ldots, w^{(n)}$.
    Notice that the final rewards of both $\tilde \pi$ and $\pi$ will have the term $G:=\prod_{i=1}^{n-1} g_i ( \dist(w^{(i)}, w^{(i+1)}) )$. We abbreviate the term as $G$ so that we can focus on comparing the remaining terms.

    Suppose $\tilde \pi$ terminates in the $(n+l)$-th round. In particular, assume it terminates on the state $\tilde s$ with round history $w^{(1)}, \ldots, w^{(n)}, \tilde w^{(n+1)}, \ldots, \tilde w^{(n+l)}$, free variables $\tilde S$ and terminal assignment $\tilde w$. For notational convenience, we will denote
    $\tilde w^{(n+l+1)} = \ext\lp(\tilde w, \tilde S \rp)$.\footnote{Notice it could be that the terminal state $\tilde s$ is in the same round as $s_{\text{curr}}$. In that case, we have $l=0$.}
    Then, the value of $\tilde \pi$ starting from $w_{\text{curr}}$ (also the reward of the state $\tilde s$)
    can be written as
    \begin{align} \label{eq:tilde-pi-value}
        V^{\tilde \pi}(s_{\text{curr}}) = G \cdot \prod_{i={n}}^{n+l} g_{i} \lp( \dist(\tilde w^{(i)}, \tilde w^{(i+1)}) \rp) \cdot g_{n+l+1} \lp(\dist\lp(\tilde w^{(n+l+1)}, w^* \rp)  \rp). 
    \end{align}
    First, we argue that it is never beneficial for $\tilde \pi$ to terminate in rounds after the $(n+1)$-th round. More formally, we will show
    \begin{align} \label{eq:path-compression}
    \prod_{i={n}}^{n+l} g_{i}( \dist(\tilde w^{(i)}, \tilde w^{(i+1)}) ) \cdot g_{n+l+1}(\dist(\tilde w^{(n+l+1)}, w^*))
    \leq g_n\lp(  \tilde w^{(n)}, \tilde w^{(n+1)} \rp)
    \cdot g_{n+1} \lp(  \tilde w^{(n+1)},w^* \rp). 
    \end{align}
Using Claim \ref{clm:monotone}, we have 
    \begin{align*}
         & g_{n+l}\bigg( \dist(\tilde w^{(n+l)}, \tilde w^{(n+l+1)}) \bigg) \cdot
        g_{n+l+1}\bigg( \dist(\tilde w^{(n+l+1)}, w^*) \bigg)            \\
         & \leq
        g_{n+l}\bigg(\dist(\tilde w^{(n+l)}, \tilde w^{(n+l+1)}) + \dist(\tilde w^{(n+l+1)}, w^*)\bigg) \cdot g_{n+l+1}(0)
        \leq g_{n+l}\bigg(\dist(\tilde w^{(n+l)}, w^*)\bigg) \, ,
    \end{align*}
    where the last inequality follows from $g_{n+l+1}(0) = 1$, the triangle inequality used with $\dist(\cdot)$ and that $g_{n+l}(\cdot)$ is monotonically decreasing (\Cref{clm:range}).
    This then shows that
    $$
    \prod_{i={n}}^{n+l} g_{i}( \dist(\tilde w^{(i)}, \tilde w^{(i+1)}) ) \cdot g_{n+l+1}(\dist(\tilde w^{(n+l+1)}, w^*))
    \leq 
    \prod_{i={n}}^{n+l-1} g_{i}( \dist(\tilde w^{(i)}, \tilde w^{(i+1)}) ) \cdot g_{n+l}(\dist(\tilde w^{(n+l)}, w^*)).
    $$
    We can then do induction on $l$ to get Equation~\eqref{eq:path-compression}. Substituting Equation~\eqref{eq:path-compression} into Equation~\eqref{eq:tilde-pi-value} then gives
    \begin{align} \label{eq:path-simplify}
        V^{\tilde \pi}(s)
        \leq G \cdot g_n\lp(\dist\lp( w^{(n)}, \tilde w^{(n+1)}\rp) \rp) \cdot g_{n+1}\lp(\dist \lp(\tilde w^{(n+1)}, w^*\rp)\rp).
    \end{align}
    We then proceed to argue the expression above is upper bounded by $V^{\pi}$.  
    Notice that by the triangle inequality, for any $\tilde w^{(n+1)}$ 
     it holds that
    $$
        \dist(w_{\text{curr}}, \tilde w^{(n+1)}) + \dist(\tilde w^{(n+1)}, w^*) \geq \dist(w_{\text{curr}}, w^*)\,.
    $$ 
    On the other hand, we always have
    $$
    \dist(w_{\text{curr}}, w^*) = \efixc{w_{\text{curr}}, w^*} + \efreec{w_{\text{curr}}, w^*}.
    $$
    Combining the two and rearranging the terms then gives
    \begin{align}\label{eq:esfix-triangle}
        \dist(w_{\text{curr}}, \tilde w^{(n+1)}) + \dist(\tilde w^{(n+1)}, w^*) - \efreec{w_{\text{curr}}, w^*}
        \geq \efixc{w_{\text{curr}}, w^*}
    \end{align}
    Now, we will use case analysis based on the relative sizes of $\dist(w_{\text{curr}}, \tilde w^{(n+1)})$ and $\efreec{w_{\text{curr}}, w^*}$.
    We first consider the case $\dist(w_{\text{curr}}, \tilde w^{(n+1)}) \leq \efreec{w_{\text{curr}}, w^*}$.
    In this case we have 
    \begin{align*}
    &g_n\Big( \dist(w^{(n)}, \tilde w^{(n+1)})\Big) \cdot
        g_{n+1}\Big( \dist(\tilde w^{(n+1)}, w^*)\Big)\\
    &= g_n\Big( \dist(w^{(n)}, w_{\text{curr}}) + \dist( w_{\text{curr}} , \tilde w^{(n+1)})\Big) \cdot
    g_{n+1}\Big( \dist(\tilde w^{(n+1)}, w^*)\Big) \\
    &\leq g_n\Big( \dist(w^{(n)}, w_{\text{curr}}) + \efreec{w_{\text{curr}}, w^*} \Big) \\
    &\cdot g_{n+1} \Big( \dist (\tilde w^{(n+1)}, w^*) + \dist(w_{\text{curr}}, \tilde w^{(n+1)}) - \efreec{w_{\text{curr}}, w^*}  \Big)\\
    &\leq
    g_n\Big( \dist(w^{(n)}, w_{\text{curr}}) + \efreec{w_{\text{curr}}, w^*} \Big) \cdot g_{n+1}\Big( \efixc{w_{\text{curr}}, w^*}  \Big) \, ,
    \end{align*}
    where the first inequality follows from Claim \ref{clm:monotone} and the second inequality follows from $g_n$ is a monotonically decreasing  function (\Cref{clm:range})  and Equation \eqref{eq:esfix-triangle}.

    Now, we consider the other remaining case when $\dist(w_{\text{curr}}, \tilde w^{(n+1)}) > \efreec{w_{\text{curr}}, w^*}$.
    Denote $M$ as the set of free variables on which $w_{\text{curr}}$ and $w^*$ agree but $w_{\text{curr}}$ and $\tilde w^{(n+1)}$ disagree.
    In other words, these are the variables mistakenly flipped by the policy $\tilde \pi$ on the path from $w_{\text{curr}}$ to $\tilde w^{(n+1)}$.
    Since $\tilde w^{(n+1)} , w^*$ disagree on these variables, these variables must be flipped again on the path from $\tilde w^{(n+1)}$ to $w^*$.
    We can then consider the alternative path $w_{\text{curr}} \rightarrow \bar w \rightarrow w^*$ for $\bar w$ satisfying
    $\bar w_i \neq \tilde w^{(n+1)}_i$ for $ i \in M$ and $\bar w_i = \tilde w^{(n+1)}_i$ for $ i \not \in M$.
     Then, it is easy to see that 
    \begin{align*}
    &g_n(\dist(w^{(n)}, w_{\text{curr}}) + \dist(w_{\text{curr}}, \tilde w^{(n+1)})) \cdot g_{n+1}(\dist (\tilde w^{(n+1)}, w^*))\\
    &\leq
    g_n(\dist(w^{(n)}, w_{\text{curr}}) + \dist(w_{\text{curr}}, \bar w)) \cdot g_{n+1}(\dist (\bar w, w^*))        
    \end{align*}
    since $g_n(\cdot)$ is monotonically decreasing. 
    Moreover, now we have $\dist(w_{\text{curr}}, \bar w) \leq \efreec{w_{\text{curr}}, w^*}$ since the variables flipped are restricted to be the ones on which $w_{\text{curr}}$ and $w^*$ do not agree. Hence, the proof is reduced to the first case.

\end{proof}

\subsection{RL algorithm to SAT algorithm}
Following the approach taken in previous lower bound \cite{kane2022computational}, we now build a randomized algorithm $\alg_{SAT}$ for \sat~using a randomized algorithm $\alg_{RL}$ for the RL problem.
In particular, we build an ``approximate'' simulator $\bar M_{\varphi}$ for the MDP oracle $M_\varphi$. The simulator $\bar M_{\varphi}$ is exactly the MDP $M_\varphi$ in terms of the transition function and features associated with the MDP $M_\varphi$, but differs in the reward function at the last layer which is always $0$ for the simulator $\bar M_{\varphi}$.
With the purposed modification, we can execute each call to simulator $\bar M_{\varphi}$ in time $\text{poly}(d)$.

\paragraph{Algorithm.} On input \cnf~formula $\varphi$, $\alg_{SAT}$ runs the algorithm $\alg_{RL}$ replacing each call to MDP oracle $M_{\varphi}$ with the corresponding call to simulator $\bar M_{\varphi}$.
Recall that the output for the RL algorithm in our setting (deterministic transition MDP) is a sequence of actions. If the sequence of actions returned by $\alg_{RL}$ ends on a state with an assignment $w$ that satisfies more than $(1-\eps)$-fraction of the clauses, $\alg_{SAT}$ terminates the simulation immediately and outputs YES. If $\alg_{RL}$ throughout the simulation never finds any state associated with such an assignment, $\alg_{SAT}$ outputs NO.

\paragraph{Correctness.}  To complete our reduction, we will show the following:  \begin{enumerate}
    \item[(i)] If algorithm $\alg_{RL}$ outputs a policy $\pi$ such that $V^\pi > V^* - 1/8$, then $\alg_{SAT}$ on \cnf~formula $\varphi$ outputs YES if $\varphi$ is satisfiable and NO otherwise.
    \item[(ii)] If $\alg_{RL}$ with access to MDP oracle $M_\varphi$ outputs a policy $\pi$ such that $V^\pi > V^* - 1/8$ with error probability $1/10$, then $\alg_{RL}$ with access to simulator $\bar M_\varphi$ outputs a policy $\pi$ such that $V^{\pi} > V^* - 1/8$ with respect to $M_\varphi$ with error probability $1/8$ (namely, even though $\alg_{RL}$ is interacting with the simulator $\bar M_{\varphi}$, the returned policy is guaranteed to do well on the true MDP $M_{\varphi}$).
\end{enumerate}
Recalling that if $\phi$ is not satisfiable, any policy is optimal, the above two claims establish
that $\alg_{SAT}$ solves \gsat~with error probability $\leq 1/8$. We start by proving that if $\alg_{RL}$ succeeds on MDP $M_\varphi$, then $\alg_{SAT}$ succeeds on \cnf~formula $\varphi$. This follows from the fact that any good policy in the MDP $M_\varphi$ must reach a state with the assignment $w^*$, the satisfying assignment which is arbitrarily chosen to construct $M_\varphi$.

\begin{proposition} \label{prop:rl-sat}
Assume that $\alpha,b,\eps$ are constants and that $v$ is large enough. Then, 
    if $\varphi$ is satisfiable and
    $\alg_{RL}$ running on $M_\varphi$ returns a policy $\pi$ satisfying $V^\pi > V^* - 1/8$ 
    then $\pi$ ends on an assignment that satisfies at least a $(1-\eps)$-fraction of clauses.
\end{proposition}
\begin{proof}
    Take a satisfiable formula $\varphi$.
    The optimal value in this case is at least $1/4$. Indeed, by 
    Lemma \ref{lem:greedy-optimality}, the greedy policy is optimal, its value is $g_1 ( \dist(w, w^*) )$ 
    and thus by Claim \ref{clm:range},
    \begin{align*}
        V^* = g_1 ( \dist(w, w^*) )
        \geq \frac{1}{4} \, .
    \end{align*}

  We now argue by contraposition: Assume that $\pi$ does not end on an assignment that satisfies at least a $(1-\eps)$-fraction of clauses.
  Let $w^{(1)}$, $\dots$, $w^{(h)}$, $\bar w$ denote the sequence of assignments obtained by $\pi$:
  $w^{(1)}=w$, and $w^{(i+1)}$ is the assignment at the end of round $1\le i\le h-1$ and $\bar w$ is the final assignment.
    Recall in each round the MDP has two stages. 
    In the first stage, the agent is presented unsatisfied clauses made up of only free variables. 
    By our construction, the first stage is of length at least $\eps v / b$. 
    It follows that 
    $\dist(w^{(i)}, w^{(i+1)}) \geq \eps v/b$ since the policies are not allowed to undo any flips.
    We can then upper bound the reward obtained at the end by
    \begin{align*}
         \prod_{i=1}^{h} g_i( \dist(w^{(i)}, w^{(i+1)}))  \cdot g_{h+1}( \dist(\bar w, w^*) ) 
         \leq 
         \prod_{i=1}^{h} g_i( \eps v / b )
          \leq \lp( 1 - \frac{\eps}{6b v^{q-2}} \rp)^h
        \leq \exp(-c v),
    \end{align*}
    where $c=\Theta(\alpha \eps / b)$,
    the first inequality follows $\dist(w^{(i)}, w^{(i+1)}) \geq \eps v/b$, the second from Claim \ref{clm:range}, and the third  follows from $1-x \leq e^{-x}$ that holds for all $x$ and our choice of $h = \alpha v^{q-1}$. Therefore, if $V^{\pi} > V^* - 1/8\ge 1/8$, and $v$ is large enough so that $\exp(-cv)<1/8$, then the policy $\pi$ has to end on a state
    which satisfied at least a $(1-\eps)$-fraction of clauses. 
\end{proof}

Next, we show that the behavior of $\alg_{RL}$ is about the same even if it is run on the simulator $\bar M_{\varphi}$. In particular, given $\alg_{RL}$ runs in sub-exponential time and succeeds on $M_{\varphi}$, we could argue $\alg_{RL}$ will be provided about the same information when it is executed on $\bar M_{\varphi}$ and on $ M_{\varphi}$ and therefore would succeed on the outputs of simulator $\bar M_\varphi$ albeit with a smaller constant probability.
\begin{proposition}
    \label{prop:remove-bar}
    Suppose $\alg_{RL}$ with access to MDP oracle $M_\varphi$ runs in time $T$  and outputs a policy $\pi$ such that $V^\pi > V^*  - 1/8$ with error probability $1/10$.
    Further, assume that the expected reward at the last layer of $M_{\varphi}$ is upper bounded by $1/(5T)$.
    Then $\alg_{RL}$ with access to simulator $\bar M_\varphi$, still running in time $T$, outputs a policy $\pi$ such that $V^{\pi} > V^* - 1/8$ with respect to $M_\varphi$ with error probability $1/8$.
\end{proposition}
\begin{proof}


    Let $\Pr_{M_\varphi}$ and $\Pr_{\bar M_\varphi}$ denote the distribution on the observed rewards and output policies induced by the algorithm $\alg_{RL}$ when running on access to MDP oracle $M_\varphi$ and simulator $\bar M_\varphi$ respectively.
    Let $R_i$ denote the reward received on the last layer at the end of $i$-th trajectory and $N$ be the total number of trajectories sampled by algorithm $\alg_{RL}$ when running on access to MDP oracle $M_\varphi$. By our assumption, $\alg_{RL}$ runs in time $T$ and therefore $N \leq T$.

    We remark that if the algorithm $\alg_{RL}$ ever reaches a satisfying assignment, $\alg_{SAT}$ will terminate the simulation immediately, returning YES.
    Before reaching a satisfying assignment, $\alg_{RL}$ may only receive rewards from the last layer.
    Since the expected reward at the last layer in the MDP $M_\varphi$ is upper bounded by $1/(5T)$ by our assumption,
    and the algorithm only visits at most $N \leq T$ states on last layer, we get by the union bound that with high probability all the rewards at the last level are zero. More precisely, we have
    \begin{align*}
        \Pr_{M_\varphi}\left[R_i = 0 ~\forall i \in [N]\right] \geq 1 - T/(5T) \geq \frac{4}{5}.
    \end{align*}
    We say that $\alg_{RL}$ succeeds with access to $M_\varphi$ (or $\bar M_\varphi$) if the output policy $\pi$ satisfies $V^\pi > V^* - 1/8$ with respect to $M_{\varphi}$ after running for time at most $T$. 
    Using the above reasoning and the assumption that $\alg_{RL}$ succeeds with access to MDP oracle $M_\varphi$ with probability $9/10$ implies 
    \begin{align*}
        \Pr_{M_\varphi}\left[\alg_{RL}~\text{succeeds with access to}~M_\varphi \mid R_i = 0 ~\forall i \in [N]\right] & \geq \frac{\frac{9}{10}-\frac{1}{5}}{\frac{4}{5}} =
        \frac{7}{8}.
    \end{align*}
    Note that the
    marginal
    distributions $\Pr_{M_\varphi}$ and $\Pr_{\bar M_\varphi}$ conditioned on $R_i = 0~\forall i \in [N]$ are exactly the same because MDP oracle $\bar M_\varphi$ and simulator $M_\varphi$ may only differ on last layer rewards before $\alg_{RL}$ reaches a satisfying assignment. This implies
    \begin{align*}
         & \Pr_{\bar M_\varphi}\left[\alg_{RL}~\text{succeeds with access to}~\bar M_{\varphi} \mid R_i = 0 ~\forall i \in [N]\right] \\
         & = \Pr_{M_\varphi}\left[\alg_{RL}~\text{succeeds with access to}~M_{\varphi} \mid R_i = 0 ~\forall i \in [N]\right]
    \end{align*}
    Since, $\Pr_{\bar M_\varphi}\left[R_i = 0 ~\forall i \in [N]\right] = 1$, we conclude that
    \begin{align*}
        \Pr_{\bar M_\varphi}\left[\alg_{RL}~\text{succeeds with access to}~\bar M_{\varphi}\right] & \geq  \frac{7}{8}.
    \end{align*}
\end{proof}
We next prove using standard reductions that $(b,\eps)$-~\gsat~is approximately as hard as \sat. 
\begin{proposition} \label{prop:gap-hardness}
    Under \reth, there exists constants $b,\eps,c> 0$ such that no randomized algorithm can solve $(b,\eps)$-~\gsat~ with $v$ variables in time
    $\exp( c v / \polylog(v) )$ with error probability $1/8$.
\end{proposition}
We provide a proof in \Cref{app:hardness}.
Now, we are ready to prove our main result, \Cref{prop:main}. For this, we demonstrate how one could reduce a $(b,\eps)$-~\gsat~ instance into an MDP instance.
\begin{proof}[Proof of \Cref{prop:main}]
    Set $p=2, q =4$ or $p = 2 \log v $ and $q = 2$. For any $v \in \mathbb Z^+$, suppose there exists an algorithm $\alg_{RL}$ which can solve~\linear{3}~with feature dimension $d = \Theta \lp( v^{2p} \rp)$ and $H = \Theta \lp( v^{q} \rp)$ with error probability $1/10$ and runs in time $\exp( c_1 \cdot v / \polylog(v)  )$ for $c_1< \min(1/2, c/2)$ where $c$ is the constant from \Cref{prop:gap-hardness}.
    Then, we claim we can build another algorithm $\alg_{SAT}$  which can solve $(b,\eps)$-~\gsat~with error probability $1/8$ in time $\exp( c v / \polylog(v) )$ . Note that this would contradict \Cref{prop:gap-hardness} under rETH and hence prove our proposition.

    Let $\varphi$ be the 3-CNF formula of a $(b, \eps)$-~\gsat~instance containing $v$ variables and at least $v$ clauses. Then, by definition, each variable appears in at most $b$ clauses.
    Furthermore, $\varphi$ is guaranteed to either be satisfiable or that at least an $\eps$-fraction of the clauses are not satisfiable under any assignment.
    To decide between the two cases, we first build an MDP $M_\varphi$ (parameterized by the two positive integers $p,q$) as described in Section~\ref{sec:mdp-construction}. In particular, the MDP is designed to have $\alpha v^{q-1}$~rounds and the polynomials $g_i$ will be a degree-$p$ Taylor approximations as specified in Equation~\eqref{eq:sub-reward-log}. As $\alpha,b,\epsilon$ are absolute constants, we ignore the dependence on them below.

    We will proceed to bound the time horizon and the feature dimension of $M_{\varphi}$ respectively.
    Since each~round consists of $v$ steps, the
    horizon is $H = \Theta( v^q )$.
    Furthermore, by \Cref{lem:greedy-reward}, the value function for the greedy policy can be written as a linear function of a feature vector of size $\Theta \lp(v^{2p} \rp)$.
    By \Cref{lem:greedy-optimality}, the greedy policy is optimal. Hence, the feature dimension of the MDP is  $d = \Theta\lp(v^{2p}\rp)$.

    Next, as noted in the proof of \Cref{prop:rl-sat}, for any policy $\pi$ which terminates on the last level, the expected reward is always upper bounded by $\exp(-v)$. Let $\bar M_{\varphi}$ be the MDP that differs from $M_{\varphi}$ only with respect to the rewards received at the end of the horizon (the rewards of $\bar M_{\varphi}$ are consistently $0$).
    Then, by \Cref{prop:remove-bar} and small $\exp(-v)$ reward noted above, we know $\alg_{RL}$, when ran for at most $\exp(c_1 \cdot v / \polylog(v))$ time (as $c_1 < 1/2$) on the simulator of $\bar M_{\varphi}$, will still output a good policy $\tilde \pi$ with respect to $M_{\varphi}$ with probability at least $7/8$.

    By \Cref{prop:rl-sat}, if $\alg_{RL}$ succeeds and $\varphi$ is satisfiable, then the policy $\tilde \pi$ will terminate on a satisfying assignment.
    Hence, we can just check the path obtained by running policy $\tilde \pi$ to decide whether $\varphi$ is satisfiable, which takes at most $\poly(v)$ time.

    Hence, the existence of such an algorithm $\alg_{RL}$ which runs in time at most $\exp(c_1 \cdot v / \polylog(v))$ time implies the existence of another algorithm which can solve the $(b,\eps)$-~\gsat~ problem in time $\exp(c_1 v / \polylog(v)) + \poly(v) \leq \exp( c v / \polylog(v) )$.
\end{proof}

\bibliographystyle{alpha}
\bibliography{main}
\newpage
\appendix
\section{Hardness of Approximate SAT with gap and few clauses}
\label{app:hardness}
In this section, we prove the following:
\begin{proposition}
    \label{prop:hard-app}
    Under \reth, there exists constants $b,\eps,c> 0$ such that no randomized algorithm can solve $(b,\eps)$-~\gsat~with $v$ variables in time
    $\exp( c v / \polylog(v) )$ with error probability $1/8$.
\end{proposition}

To prove this, we will look at another problem: $\eps$-~\gsat. This is similar to  $(b,\eps)$-~\gsat~except it does not put any constraints on how many clauses a variable can be in. Through standard technique, one can show that $\eps$-~\gsat~ is also hard. In particular, its hardness is shown in  \cite{moshkovitz2008two} and relies on a certain version of the Probabilistic Checkable Proof (PCP) theorem.
\begin{figure}
    \begin{center} \label{def:approx-gap-sat}
        \begin{minipage}{0.9\textwidth}
        \hrulefill\\[5pt]
        \textbf{Complexity problem}\quad $\eps$-~\gsat~ \\[3pt]
        \begin{tabular}{@{}p{0.1\linewidth} @{}p{0.89\linewidth}}
             \textit{Input:} &  A gap parameter $\eps > 0$ and a \cnf~formula $\varphi$ with $v$ variables and $O(v)$ clauses such that the either (i) $\varphi$ is satisfiable or (ii) any assignment leaves at least an $\eps$-fraction of the clauses unsatisfied where $\eps > 0$. \\
             \textit{Goal:} & Decide whether the formula is satisfiable.
        \end{tabular} \\[5pt]
        \vphantom{.}\hrulefill
        \end{minipage}
    \end{center}
\end{figure}

\begin{theorem}[Reduction from \sat~to \gsat] \label{thm:inapproximate-gsat}
Solving \sat~on inputs of size $n$ can be reduced to distinguishing
between the case that a \cnf~formula of size $n \cdot \polylog(n)$ is satisfiable and the case that only $1 - \eps$
fraction of its clauses are satisfiable for some constant $\eps > 0$.
\end{theorem}
For completeness, we provide a proof for the above theorem. We first review some basic concepts about the PCP theorem. Given a statement (for example, whether a SAT instance is satisfiable), a PCP verifier is granted query access to a proof constructed for the statement over an alphabet $\Sigma$ and asked to decide whether the statement is true. A PCP verifier has several important parameters. 
\begin{itemize}
    \item \textbf{Completeness} $c$: The minimal probability that the verifier accepts a correct proof.
    \item \textbf{Soundness} $\eps$: The maximal probability that the verifier accepts a proof for an incorrect theorem.
    \item \textbf{Queries} $q$: The number of queries made by the verifier to the proof.
    \item \textbf{Size} $m$: The length of the proof.
    \item \textbf{Randomness} $r$: The number of random bits used by the verifier.
    \item \textbf{Alphabet} $\Sigma$: The alphabet used by the proof.
\end{itemize}
We denote by $\PCP_{c,\eps}[r,q]_{\Sigma}$ the class of languages that have a PCP verifier with completeness $c$, soundness $s$, randomness $r$, and $q$ queries to a proof over alphabet $\Sigma$. Moreover, the PCP verifier is only allowed to do a two query projection test. In a two query-projection test, the verifier is only allowed to make two queries. Upon seeing the answer to the first query, the verifier either immediately rejects, or it has uniquely determined answer to the second query on which it accepts. Our starting point is the following theorem from \cite{moshkovitz2008two}.
\begin{theorem}[Theorem 7 from \cite{moshkovitz2008two}]
There exists a constant $\eps > 0$ and an alphabet $\Sigma$ of constant size, such that $\sat~\in PCP_{1, 1- \eps}[\log n + O(\log \log n), 2]_{\Sigma}$.
\end{theorem}

\begin{proof}[Proof of \Cref{thm:inapproximate-gsat}]
Given a \cnf~formula $\psi$ with size $n$, the goal is to use the verifier in the above theorem to construct a different \sat~instance $\phi$ with size $O(n \polylog n)$ such that
(i) $\phi$ is satisfiable if $\psi$ is satisfiable; and
(ii) at least an $\eps$ fraction of the clauses in $\phi$ are not satisfiable under any assignment if $\psi$ is not satisfiable.

Notice that we can without loss of generality assume the verifier is deterministic if we assume it also takes $ r \eqdef \log n + O(\log \log n)$ random bits as input. Fix a random bit string, the verifier reads at most $2$ characters from the proof. Since there are at most $2^r = n \polylog n$ different random bit strings, we can without loss of generality assume the proof is of size at most $T \eqdef 2 n \polylog n$. 
The first step of the construction is to  create $T$ variables $\{x_1, x_2, \cdots, x_T\}$ where $x_i \in \Sigma$ represent the queries responses given to the verifier.
We will create a SAT formula for each of the $2^r$ random bit strings and the final construction will be simply the concatenation of all the SAT formulas with the \emph{``AND''} logical operator.
Fix an arbitrary random string $q \in \{0, 1\}^r$. We can then compute the first position the verifier will read. We can denote it as $l_1(q)$. There will be a subset of values $R(q) \subseteq \Sigma$ that the verifier will reject immediately $x_{l_1(q)} \in R(q)$.
If the verifier does not reject immediately, the verifier could branch off to do different things based on the value of $x_{ l_1(q) }$. Suppose, $x_{ l_1(q) } = \sigma \in \Sigma \backslash R(q)$. We can then compute the second position the verifier will read, which we denote as $ l_2(q, \sigma  )$, and the ``right'' character the verifier is expecting, which we denote as $ f( q, \sigma  )  $. Then, we know that the verifier will accept if and only if the proof, represented by $x_1, \cdots, x_T$ satisfies that
$$
\bigcup_{\sigma \in \Sigma \backslash R} 
\lp( x_{ l_1(q) } = \sigma \rp) \land 
\lp(  x_{l_2(q, \sigma)} = f(q, \sigma) \rp) .
$$

Since $\Sigma$ is of constant size, it is easy to see that one can use a binary encoding for $\Sigma$ and convert the above statement into a \cnf~formula of constant size.
In addition, there are at most $n \polylog n$ binary strings $q$. Hence, the overall \cnf~formula $\phi$ is of size $O(n \polylog n)$. We know that the verifier would reject with probability at least $\eps$ if the original sat instance $\psi$ is not satisfiable. 
Hence, at least an $\eps$ fraction of the sub-formulas of $\phi$ will not be satisfied under any assignment (which can be interpreted as the binary encoding of the given proof).  
On the other hand, if $\psi$ is satisfiable, it then holds every sub-formula of $\phi$ is satisfiable since the verifier always accepts under the ``correct'' proof.

If one has an algorithm which can distinguish between the cases that $\eps$-fraction of $\phi$ cannot be satisfied under any assignment versus $\phi$ is satisfiable, one can then decide the satisfiability of $\psi$. 
\end{proof}

We are interested in $(b, \eps)-$~\gsat, which is a restricted version of $\eps$-~\gsat, where each variable is promised to appear in at most $b$ clauses for some constant $b$. One can show that approximating $(b, \eps)-$~\gsat~is also hard through a reduction given in \cite{papadimitriou1991optimization}.
\begin{proposition}[Adapted from Proof of Theorem 2 in \cite{papadimitriou1991optimization}] \label{prop:bgsat-reduction}
For some constant integer $b=O(1)$, there is a polynomial time transformation which maps a 3-CNF formula $\phi$ to another 3-CNF formula $\psi$ over the same set of variables such that 
\begin{enumerate}
    \item Each variable appears in at most $b$ clauses in $\psi$.
    \item If $\phi$ is satisfiable, then $\psi$ is also satisfiable.
    \item Let $\abs{\phi}, \abs{\psi}$ denotes the number of clauses in $\phi$, $\psi$ respectively. Then, $\abs{\phi} \leq \abs{\psi} \leq O(1) \cdot \abs{\phi}$. 
    \item Let $\max(\phi), \max(\psi)$ denote the maximum number of clauses satisfiable in $\phi$ and $\psi$ respectively. It holds $ \max(\psi) \leq  \max(\phi) + \abs{\psi} - \abs{\phi}$.
\end{enumerate}
\end{proposition}
\begin{proof}[Proof of \Cref{prop:hard-app}]
    Proposition \ref{prop:bgsat-reduction} states there is an efficient algorithm translating an $\eps$-~\gsat~ instance consisting of $m$ clauses into an $(b, \alpha \cdot \eps )$-~\gsat~ instance for some constant $\alpha \in (0,1)$. Hence, if there is no sub-exponential algorithm for the computational problem $\eps_1$-~\gsat~ for some constant $\eps_1 \in (0,1)$, there is no sub-exponential algorithm for $(b, \eps_2)$-~\gsat~ either  for some constant $b$ and $\eps_2 \in (0, \eps_1)$.
\noindent Combining Theorem \ref{thm:inapproximate-gsat} and Proposition \ref{prop:bgsat-reduction} proves our claim.
\end{proof}

\section{Upper Bounds} \label{appendix:upper}
In both of the upper bounds, the final policy computed by our algorithms is of the following form: at the state $s$, we have some estimations $ \tilde Q(s,a) $ for each $a \in \mathcal A$ such that 
$ \abs{\tilde Q(s,a) - Q^*(s,a)} \leq \eps $ and the policy always chooses the action $a = \argmax_a \tilde Q(s,a)$. We claim the policy induced is nearly optimal as long as $\eps$ is sufficiently small. The formal statement is given below.
\begin{lemma} \label{lem:Q-reduction}
For any state action pair $(s,a)$,
let $\tilde Q(s,a)$ be an approximation of $Q^*(s,a)$ satisfying
$ \abs{ \tilde Q(s,a) - Q^*(s,a) } \leq \eps/(2H)$. Then, consider the policy $\pi$ such that at the state $s$, it always chooses the action $a = \argmax_a \tilde Q(s,a)$.
Then, it holds 
$V^{\pi}(s) \geq V^*(s) - \eps$ for any state.
\end{lemma}
\begin{proof}
  We claim $\pi$ is a policy satisfying that
   $V^{\pi}(s) \geq V^{*}(s) - \eps \cdot h/H $ for any state $s$ in the MDP such that there are still $h$ steps remaining.
   We show this via induction on the number of steps remaining.
   Suppose $s$ is a state right before the last step. 
   Then, $V^*(s) =  Q^*(s,a^*)$ for $a^* = \argmax_a Q^*(s,a^*)$
   and
   $V^{\pi}(s) = Q^*(s, a')$ for $a' = \argmax_a \tilde Q(s,a) $.
   By our assumption, we have
   $V^{\pi}(s) \geq \tilde Q(s,a') - \eps/(2H)$ and since $a' = \argmax_a \tilde Q(s,a) $, we then further have
   $$
   V^{\pi}(s) \geq \tilde Q(s,a^*) - \eps/(2H)
   \geq Q^*(s, a^*) - \eps/H = V^*(s) - \eps/H.
   $$
   Now, consider a state $s$ such that there are $(h+1)$ steps remaining. 
   Still, let $a' = \argmax_{a} \tilde Q(s,a)$ and $a^* = \argmax_a Q^*(s,a^*)$.
   Furthermore, let $s' = P(s, a')$ be the next state after applying $a'$.
   We then have $V^*(s) =  Q^*(s,a^*)$ and 
   $V^{\pi}(s) = \E \lp[ R(s, a') \rp] 
   + V^{\pi}(s')
   $.
   We then have
   \begin{align*}
    V^{\pi}(s)
    &= \E \lp[ R(s, a') \rp] 
   + V^{\pi}(s') &\text{(Definition of the policy $\pi$)}\\
   &\geq \E\lp[ R(s, a') \rp] + V^{*}(s') - \eps h /H
   &\text{(Inductive Hypothesis)} \\
   &= Q^*(s, a') - \eps h/H &\text{(Definition of $Q^*, V^*$)} \\
   &\geq \tilde Q(s,a')  - \eps h /H- \eps/(2H) &\text{(Assumption about $\tilde Q$)}\\
   &\geq  \tilde Q(s,a^*) - \eps h/H - \eps/(2H) &\text{(Choice of $a' = \argmax_a \tilde Q(s,a)$)}\\
   &\geq Q^*(s, a^*)  - \eps h/H - \eps/H
   &\text{(Assumption about $\tilde Q$)}
   \\   
   &= V^*(s)  - \eps (h+1)/H &\text{(The choice of $a^* = \argmax_a Q^*(s,a)$ and the definition of $Q^*, V^*$)}.
   \end{align*}
   This then gives us $V^{\pi}(s) \geq V^*(s) - \eps$ for any state since there are in total $H$ steps in the MDP.
\end{proof}

We first prove a computational upper bound which is exponential in the feature dimension $d$. 
On a high level, we discretize the parameter space that $\theta^*$ may lie in to create a policy cover which allows us to search for the best in class by estimating the value of each policy.
\begin{proposition}
    Assume the Linear MDP has a constant number of actions, feature dimension $d$ and time Horizon $H$. 
    Furthermore, assume the featuer vectors satisfy $\snorm{2}{\psi(s,a)} \leq 1$ for all state action pairs and $
    \snorm{2}{\theta^*} \leq 1$ for the optimal parameter $\theta^*$.
    Let $\eps \in (0,1)$.
    There is an algorithm which takes $\exp\lp( c \cdot d \cdot \log \lp(  H d / \eps \rp) \rp)$
    time for some sufficiently large constant $c$ and finds a policy $\pi$ such that $V^{\pi}(s) \geq V^*(s) - \varepsilon$ with probability $9/10$.
\end{proposition}
\begin{proof}
   Let $\theta^*$ denote the unknown parameters of the optimal $Q^*$ function, i.e.
   $ Q^*(s,a) = \langle \theta^*, \psi(s,a) \rangle  $. 
   Suppose we can find such a $\theta \in \R^d$ satisfying 
   $\snorm{2}{\theta - \theta^*} \leq  \varepsilon/(2 H \cdot \sqrt{d})$. 
     We note that this implies 
     \begin{align} \label{eq:Q-approx}
   \abs{\langle \theta, \psi(s,a)\rangle - Q^*(s,a)} \leq \varepsilon/(2H).
     \end{align}
    Then, consider the policy $\pi(\theta)$ such that at state $s$ it always chooses the action
    $a = \argmax_{a} \langle \theta, \psi(s,a) \rangle.$
    By \Cref{lem:Q-reduction}, it holds
    $V^{ \pi(\theta) }(s) \geq
    V^*(s) - \eps
    $ for any state $s$.

   Now, let $S \in \R^d$ be the set of vectors that form an $\eps/(2H\sqrt{d})$-cover of the $d$-dimensional unit sphere, i.e. $ \min_{ \theta \in S } \snorm{2}{ \theta - \theta^* } \leq  \eps/(2H\sqrt{d})$ for any $\theta^*$ satisfying $\snorm{2}{\theta^*} \leq 1$. Through a standard combinatorial construction, there exists such a cover $S$ with size $\abs{ S } \leq \exp\lp( c \cdot d \cdot \log \lp(  H d / \eps \rp) \rp)$ for some sufficiently large constant $c$.
   From the argument above, we know there must be some $\theta \in S$ such that  $\pi(\theta)$ is nearly optimal, i.e.
   $V^{ \pi(\theta) }(s) \geq V^*(s) - \eps$ for any state $s$.
   
   Our strategy is simple: we will try $\pi(\theta)$ for all $\theta \in S$ in a brute-force manner and estimate the expected reward of the induced trajectory up to accuracy $\eps$.
   Notice that the maximum reward collected by any trajectory is at most $H$.
   Hence, if we visit the same trajectory with $\poly(H, 1/\eps) \cdot \log(1/\delta)$ many times, we can then compute an estimation of its expected reward up to accuracy $\eps$ with probability at least $1-\delta$. We can take $\delta = \frac{1}{|S|}$ so that by union bound our estimation for $V^{ \pi(\theta) }(s_0)$ is accurate up to error $\eps$ for all $\theta \in S$ with probability at least $9/10$.
   Condition on that, we can then choose $\theta$ such that it maximizes our empirical estimations of $V^{ \pi(\theta) }(s_0)$. Then, it is easy to see that such a $\pi(\theta)$ must satisfy $V^{ \pi(\theta) }(s_0) \geq V^{*}(s_0) - 2 \eps$.
   Now, since to simulating the interaction of one trajectory takes time at most $\poly(d,H)$, the total runtime is bounded by
   $$
   \poly(d,H) \cdot \poly(H, 1/\eps) \cdot \log(1/|S|)
   \cdot \abs{S}
   \leq \exp\lp( c \cdot d \cdot \log \lp(  H d / \eps \rp) \rp)
   $$
   for some sufficiently large constant $c$.
\end{proof}
To prove a horizon upper bound, we build on results of previous work \citep{du2020agnostic}. This upper bound was originally personally communicated to the authors by Ruosong Wang. We only add it here for completeness. We first give a high level overview of the differences. The proof is almost exactly the same except we now divide the steps of the MDP into $\sqrt{H}$ ``rounds''. 
We will brute force search in the rounds for the optimal policy and use the basis constructed in previous work \citep{du2020agnostic} to ensure error only grows by a factor of $\sqrt{d}$. 
We next prove this in more detail.
\begin{proposition}[Ruosong Wang, personal communication]
    \label{prop:upper2}
    Assume the Linear MDP has a constant number of actions, feature dimension $d$ and time Horizon $H$. 
    Furthermore, assume the featuer vectors satisfy $\snorm{2}{\psi(s,a)} \leq 1$ for all state action pairs and $
    \snorm{2}{\theta^*} \leq 1$ for the optimal parameter $\theta^*$.
    Let $\eps \in (0,1)$.
    There is an algorithm which takes $\exp\lp (c \cdot \sqrt{H} \log d \rp) / \eps^{-2}$ time for some sufficiently large constant $c$ and finds a policy $\pi$ such that $V^{\pi}(s) \geq V^*(s) - \varepsilon$.
\end{proposition}
\begin{proof}
    Given an arbitrary state $s$, 
    suppose there is a procedure that runs in time 
    $$
    T \eqdef \exp\lp (c \cdot \sqrt{H} \log d \rp) / \eps^{-2} \cdot \log(1/\delta)
    $$ and computes 
    an estimation of $\tilde Q(s,a)$ for each action $a \in \mathcal A$ such that 
    $ \abs{\tilde Q(s,a) - Q^*(s,a)} \leq \eps/(2H)$ with probability at least $1 - \delta$. 
    Then, we claim we can design an algorithm which outputs a policy $\pi$
    such that $V^{\pi}(s_0) \geq V^*(s_0) - \eps$ for the initial state $s_0$ with probability at least $9/10$.
    Starting at the state $s = s_0$, we will perform the following steps iteratively:
    \begin{enumerate}
        \item For the current state $s$,  compute the estimations $\tilde Q(s,a)$.
        \item Choose $a = \argmax_{a} \tilde Q(s,a)$ and then updates $s$ to be the next state after applying action $a$.
    \end{enumerate}
    The above process goes on for at most $H$ iterations. Hence, our estimations $\tilde Q(s,a)$ are accurate in all iterations with probability at least $9/10$ if we set $\delta = 1 / (10H)$.
    By \Cref{lem:Q-reduction}, it then holds the resulting policy is nearly-optimal starting from the initial state $s_0$.
    Moreover, the algorithm runs in time $
    H \cdot O(T)
    $, which is within the desired runtime.
    
    To finish the proof, we then describe our procedure for computing the estimations $\{\tilde Q(s,a)| a \in \mathcal A\}$ for a state $s$. 
    We will describe the procedure for just the initial state $s_0$ as computing the estimations for other states can be done similarly.
    To do so, we divide the time steps of the MDP into $\sqrt{H}$ rounds. 
    For each round $h \in \sqrt{H}$, we build a set of  vectors $B_h$ that correspond to the ``basis'' of some larger set of feature vectors $\psi(s,a)$ where $s$ is a state on the $h \cdot \sqrt{H}$ level, i.e. there is a trajectory going from $s_0$ to $s$ in $h \cdot \sqrt{H}$ steps.
    The step is similar to previous work \cite{du2020agnostic} and proceeds as follows. 
    Let $B_0 = \{\psi(s_0,a): a \in A\}$ where $a$ is the set of all actions. Then, we construct $B_h$ recursively from $B_{h-1} = \{\psi(s_i, a_i)\}$: Let $\bar B_h = \{\psi(P(s_i, a_i), a): a \in A~\text{and}~ \psi(s_i, a_i) \in B_{h-1}\}$. Note that $|\bar B_h| \leq |\mathcal A| \cdot |B_{h-1}|$. Next, we set $B_h \subset \bar B_{h}$ as any maximal subset of independent vectors of $\bar B_{h}$. Note here $|B_h| \leq d$. Moreover, since $\snorm{2}{\psi(s,a)} \leq 1$ by assumption, any $\psi(s,a) \in \bar B_{h}$ can then be written as $ \sum_{i=1}^{|B_{h}|} \alpha_i \cdot \psi(s_i, a_i)$ satisfying $\snorm{2}{\alpha} \leq \sqrt{d}$ where $\psi(s_i, a_i)$ are the base vectors in $B_{h}$.

    We claim that for any $h \in [\sqrt{H}]$, we can learn $Q^*$ on the basis $B_{h}$ to accuracy $(2d)^{-h} \eps$ using at most $O \lp(\exp(  \sqrt{H} \cdot \log d ) \cdot \eps^{-2} \rp)$ time.
    We show this via induction on $h$.
    Notice that for any state $s$ on the last level (which are $1$ step from termination) and action $a \in \mathcal A$, the function $Q^*(s,a)$ is simply the expected reward $\E \lp[ R(s,a) \rp]$ since the MDP terminates immediately afterwards. Hence, we can follow the same trajectory and sample from $R(s,a)$ for multiple times and compute an empirical mean $\bar R(s,a)$. Suppose we take $C \cdot \varepsilon^{-2}  \log (H \cdot |B_H|/\delta) \cdot (2d)^{ \sqrt{H} }$ samples from $R(s,a)$ for a sufficiently large constant $C$. It then follows from standard concentration inequalities that 
    $ \abs{ \bar R(s,a)  - \E[ R(s,a) ]} \leq (2d)^{-\sqrt{H}} \epsilon$ with probability at least $ 1 - \delta / \lp(H \cdot |B_H|\rp) $.
    By the union bound, this holds for all $\phi(s,a) \in B_{H}$ with probability at least $1 - \delta / H$. 
    Therefore, for all $\phi(s,a) \in B_{\sqrt{H}}$, we can compute an estimator for for $Q^*(s,a)$ with accuracy $(2d)^{-\sqrt{H}}$ in time at most $$ 
    \lp(C \cdot \varepsilon^{-2}  \log (H \cdot |B_H|/\delta) \cdot (2d)^{ \sqrt{H} }\rp) \cdot |B_H| \cdot \poly(d,H) 
    \leq \exp\lp( c \cdot \sqrt{H} \cdot \log d \rp) \eps^{-2} \log(1/\delta)
    $$ for some large enough constant $c$.
    
    
    Assume we have already learned $Q^*$ on the basis $B_{h}$ to accuracy $d^{-h} \epsilon$. 
    We will see how we can use the information to estimate $Q^*$ on the basis $B_{ h-1 }$ to accuracy $d^{-h+1} \epsilon$.
    Still, consider a single state-action pair $(s,a)$ such that $\phi(s,a) \in B_{ h-1 }$. 
    Let $\mathcal R_{s, a, \sqrt{H}}$ be the set of states reachable from $s$ within $\sqrt{H}$ many steps condition on that the first step is $a$ (notice that $\mathcal R_{s, a, \sqrt{H}}$ is a subset of 
    $\bar B_{h}$ by our construction). In other words, each state $s'$ in $ \mathcal R_{s, a, \sqrt{H}}$ is a state in the  $(h \cdot \sqrt{H})$-th level such that there is a trajectory going from $s$ to $s'$ beginning with the action $a$. 
    We will without loss of generality assume that each state $s'$ has a unique trajectory starting from $s$: If there are two different trajectories leading to the same state $s'$, we can create two copies of $s'$ and index them by the unique trajectory that leads to them.
    
    We know there must exist some state $s^* \in \mathcal R_{s, a, \sqrt{H}}$ and $a^* \in \mathcal A$ such that $Q^*(s,a)$ is equal to the sum of the expected rewards collected from the trajectory from $s$ to $s^*$ and $Q(s^*, a^*)$.
    We will denote by $\kappa(s,s')$ the expected reward collected from the path going from $s$ to $s'$ for $s' \in \mathcal R_{s, a, \sqrt{H}}$.
    Our goal is then to compute 
    (i) an estimation for each $\kappa(s,s')$ where $s' \in \mathcal R_{s, a, \sqrt{H}}$
    and (ii) an estimation for each $Q^*(s', a')$ where $s' \in \mathcal R_{s, a, \sqrt{H}}$ and $a' \in A$.
    It is easy to see if we can compute both (i) and (ii) up to accuracy $ d^{-h+1} \cdot 2^{-h}  $, we can then take the optimal combination of $s', a'$ to get an estimation of $Q^*(s,a)$ up to accuracy $(2d)^{ -h+1 }$.
    
    To get an estimation of $\kappa(s,s')$, the expected reward collected from a trajectory, we can just visit the trajectory for multiple times. Since $\kappa(s, s') \leq \sqrt{H}$, it then follow from standard concentration that if we visit the trajectory for 
    $$
    C \cdot H \cdot \varepsilon^{-2}  \log (H \cdot |\mathcal R_{s, a, \sqrt{H}}|/\delta) \cdot (2d)^{ h } 
    $$
    many times where $C$ is a sufficiently large constant, then we can estimate all $\kappa(s,s')$ up to the desired accuracy
    with high probability.
    There are at most $|\mathcal A|^{\sqrt{H}} = \exp( \log |\mathcal A| \cdot \sqrt{H} )$ many states in $\mathcal R_{s,a,\sqrt{H}}$. Since we visit a trajectory   
    $    C \cdot H \cdot \varepsilon^{-2}  \log (H \cdot |\mathcal R_{s, a, \sqrt{H}}|/\delta) \cdot(2d)^{ h } $ times, estimating each $\kappa(s, s')$ takes time at most
    \begin{align*}
    &\exp( \log |\mathcal A| \cdot \sqrt{H} ) \cdot 
    \lp( C \cdot H \cdot \varepsilon^{-2}  \log (H \cdot |\mathcal R_{s, a, \sqrt{H}}|/\delta) \cdot (2d)^{h} \rp)
    \cdot \poly(d,H) \\
     &\leq \exp\lp( c \cdot \sqrt{H} \cdot \log d \rp) \eps^{-2} \log(1/\delta)    
    \end{align*}
    for some sufficiently large constant $c$.
    
    To get an estimation of $Q^*(s', a')$ where $s' \in \mathcal R_{s, a, \sqrt{H}}$ and $a' \in A$, we will take advantage of the fact that we already have estimations of $Q^*$ on the basis in $B_{ h }$. 
    In particular, we can express $\psi(s',a') = \sum_{i=1}^{ |B_{ h }| } \alpha_i \cdot \psi( s_i, a_i ) $  for $\psi( s_i, a_i )$ being the basis in $B_{ h }$. By linearity, we then have
    \begin{align} \label{eq:Q-lienar}
    Q^*(s',a') = \sum_{i=1}^{|B_{ h }|} \alpha_i \cdot Q^*(s_i, a_i).
    \end{align}
    On one hand, we have $\snorm{2}{\alpha} \leq \sqrt{d}$.
    On the other hand, by the inductive hypothesis, we  have an estimation 
    of each $Q^*(s_i, a_i)$ up to accuracy $(2d)^{-h} \cdot \eps$.
    If we simply plugin our estimation for $Q^*(s_i, a_i)$ into Equation~\eqref{eq:Q-lienar} to compute our estimation for $Q^*(s',a')$, we then have the error is at most 
    $ d \cdot  (2d)^{-h} \cdot \eps \leq 2^{-h}  \cdot d^{-h+1} \cdot  \eps$ by the Cauchy Schwarz's Inequality, which is the desired bound. Computing the estimation for one $Q^*(s',a')$ takes $\poly(d)$ time. Since there are at most $ \exp( \log |\mathcal A| \cdot \sqrt{H} ) \cdot |\mathcal A| $ many pairs of $(s',a')$, this part takes time at most $\exp\lp( c \cdot \lp( \sqrt{H} + \log d \rp) \rp)$ for some sufficiently large constant $c$.
    
    By induction, this then gives us a way to approximate $Q^*$ on $B_0 = \{ \phi(s_0, a) : a \in A \}$ up to accuracy $\eps$ with high probability. Moreover, the entire process runs in time  $\exp\lp( c \cdot \sqrt{H} \cdot \log d \rp) \eps^{-2} \log(1/\delta)$ for some sufficiently large constant $c$. 
\end{proof}
\end{document}